\documentclass[twoside,11pt]{article}

\usepackage{amsmath}
\usepackage{amssymb}
\usepackage{amsfonts}

\usepackage{psfrag}     
\usepackage{subcaption} 

\usepackage{xspace}    
\usepackage{mathtools} 

\usepackage{jmlr2e_no_editor}

\newcommand{\D}{\mathcal{D}\xspace}
\newcommand{\I}{\mathbb{I}\xspace}
\newcommand{\R}{\mathbb{R}\xspace}
\newcommand{\dx}{\text{d}x\xspace}
\newcommand{\psipdf}{\psi(x;\mu,\theta,p)\xspace}

\newtheorem{const}{Constraint}

\ShortHeadings{Asymmetric Distributions from Constrained Mixtures}{Miranda and
Von Zuben}
\firstpageno{1}

\begin{document}

\title{Asymmetric Distributions from Constrained Mixtures}

\author{\name Conrado S.\ Miranda \email conrado@dca.fee.unicamp.br \\
       \name Fernando J.\ Von Zuben \email vonzuben@dca.fee.unicamp.br \\
       \addr School of Electrical and Computer Engineering\\
       University of Campinas\\
       Campinas, SP 13083-852, Brazil}

\editor{}

\maketitle

\begin{abstract}
  This paper introduces constrained mixtures for continuous distributions,
  characterized by a mixture of distributions where each distribution has a
  shape similar to the base distribution and disjoint domains. This new concept
  is used to create generalized asymmetric versions of the Laplace and normal
  distributions, which are shown to define exponential families, with known
  conjugate priors, and to have maximum likelihood estimates for the original
  parameters, with known closed-form expressions. The asymmetric and symmetric
  normal distributions are compared in a linear regression example, showing that
  the asymmetric version performs at least as well as the symmetric one, and in
  a real world time-series problem, where a hidden Markov model is used to fit a
  stock index, indicating that the asymmetric version provides higher likelihood
  and may learn distribution models over states and transition distributions
  with considerably less entropy.
\end{abstract}


\section{Introduction}
There is a plethora of probability distributions to fit the most diverse uses.
However, even with this abundance of distributions, some applications can not be
solved using them directly, requiring the use of probabilistic
graphs~\citep{koller2009probabilistic}, like mixture
models~\citep{mclachlan1988mixture}, hidden Markov
models~\citep{baum1966statistical}, or latent Dirichlet
allocation~\citep{blei2003latent}, where a set of distributions is used to build
the joint probability distribution.

While these more complex models provide additional flexibility to describe the
problem, they are still limited by the underlying distributions used. This
motivates the search for new distributions to describe some data peculiarity,
and one of particular interest is the asymmetry of the distribution.

There are naturally asymmetric distributions, such as the lognormal
distribution~\citep{lognormal}, but it is also possible to introduce asymmetry
in symmetric distributions, like the skew normal distribution~\citep{o1976bayes}
does. This distribution is able to control the skewness of the normal
distribution, at the cost of losing closed-form expressions for the maximum
likelihood estimates. Additionally, by modifying the shape of the distribution,
its original interpretability is also lost.

To keep the interpretability, which may be important when analyzing a fitted
model, the shape of the distributions used must be maintained, such that the
user can choose the ones he or she knows how to analyze. For instance, this is
what happens with mixture models, where the known base distributions just change
their parameters and are weighted.

In this paper, we introduce the concept of a constrained mixture of
distributions for continuous distributions, which differs from the traditional
mixture in that, instead of each distribution being defined in the whole domain
and being able to overlap with the other distributions, the domain is
partitioned among the distributions. In this way, they are defined only in their
segment, and all of them are instances of the same underlying distribution with
different parameters that guarantee that the continuity of the original
distribution is kept. This allows weighting each segment and analyzing them
separately, like one would do with the distributions in a standard mixture
model.

The constrained mixture is then used to create asymmetric versions of the
Laplace and normal distributions, where the symmetric versions are particular
cases. These new distributions are shown to define an exponential family when
the partitions are known, which allows them to be easily used in existing models
designed to work with these kinds of distributions, like in latent Dirichlet
models~\citep{banerjee2007latent} and co-clustering~\citep{shan2008bayesian},
and their conjugate priors, with closed-form expressions, are also given.

We also show for these new asymmetric distributions that, if the weight of
each partition is known, then the maximum likelihood estimates are known and
their closed-form expressions are provided. Furthermore, we provide a
hill-climbing algorithm to fit the weight of the partitions, which allows
maximum likelihood estimates for all the parameters.

To show the power of introducing asymmetry to the normal distribution, two
applications are provided. The first is a simple linear regression example
problem with asymmetric noise used to gain insight into how the asymmetry
affects the estimation and show experimentally that the asymmetric likelihood is
lower bounded by the symmetric likelihood. The second is a hidden Markov model
used to fit a real world stock index time-series, which shows that the
flexibility introduced by the asymmetry not only increases the likelihood, but
may also provide insight into the system and reduce its entropy.

This paper is organized as follows. Section~\ref{sec:mixture} introduces the
concept of constrained mixtures, and the asymmetric versions of the Laplace and
normal distributions are introduced in Section~\ref{sec:asymmetric_examples}.
Section~\ref{sec:optimization} proves optimality conditions for the maximum
likelihood estimates and provides their closed-form expressions.
Section~\ref{sec:applications} compares the performance of the asymmetric normal
distribution with the symmetric version for one example and one real world
problem, showing the advantages of the new distribution. Finally,
Section~\ref{sec:conclusion} summarizes the findings and indicates future
research directions.

\section{Constrained Mixture}
\label{sec:mixture}
A constrained mixture is a special kind of mixture of distributions
characterized by the existence of only one underlying distribution so that the
domain is split in disjoint segments. Each segment has its own distribution,
which must be similar to the base distribution, that is, there are known
parameters for the base distribution that provide the shape of the distribution
in the segment. Moreover, the distributions must be continuous and the weights
for each segment must be provided.

Since a mixture of $N > 2$ distributions can be described as a mixture of 2
distributions, where one of those is a mixture of $N-1$ distributions itself,
we will develop the equations only for the base-case of $N=2$. This not only
simplifies the problem, but also is associated with the number of distributions
used to create the asymmetric versions of the Laplace and normal distributions.

\begin{definition}[Constrained Mixture]
Let $\phi(x;\theta) \colon \R \times \D(\theta) \to [0, \infty)$ be the
continuous probability density function (pdf) for some distribution $D$, where
$\D(\cdot)$ is the domain of its argument. Let $\phi_+(x;\mu,\theta) =
\phi(x;\theta) \I[x \ge \mu]$ and $\phi_-(x;\mu,\theta) = \phi(x;\theta) \I[x <
\mu]$, where $\I[\cdot]$ is the indicator function, be the partitions'
distributions. Let $p \in (0,1)$ be a weight parameter. Then the constrained
mixture $D^*$ is described by a pdf $\psipdf \colon \R^2 \times \D(\theta)
\times (0,1) \to [0,\infty)$ that satisfies the following constraints for all
$\mu$, $\theta$, and $p$ in the domain:

\begin{const}[Continuity]
  \label{const:continuity}
  The pdf is continuous at $x = \mu$, which means that
  \begin{equation*}
    \lim_{x \to \mu^+} \psipdf
    =
    \lim_{x \to \mu^-} \psipdf.
  \end{equation*}
\end{const}

\begin{const}[Mixture]
  \label{const:mixture}
  There are known functions $\Theta_\pm(\mu,\theta,p) \colon \R \times
  \D(\theta) \times (0,1) \to \D(\theta)$ and normalizing constant $Z \in
  (0,\infty)$ such that
  \begin{equation*}
    \psipdf Z = p \phi_-(x;\Theta_-(\mu,\theta,p))
    +(1-p) \phi_+(x;\Theta_+(\mu,\theta,p)).
  \end{equation*}
\end{const}
\end{definition}

Constraint~\ref{const:continuity} guarantees that the continuity of
$\phi(\cdot)$ is preserved, while Constraint~\ref{const:mixture} builds a
mixture that forces each segment of the new pdf $\psi(\cdot)$ to have the same
structure as the original pdf $\phi(\cdot)$, while also placing weight $p$ and
$1-p$ on the left and right sides of the partition, respectively. The functions
$\Theta_\pm(\cdot)$ perform the mapping from the constraint parameter $\mu$ and
$p$ and the underlying distribution parameters $\theta$ to a new set of
parameters $\Theta_\pm(\mu,\theta,p)$ that are used in each side of the
partition.

From Constraint~\ref{const:mixture} and the fact that $\psi(\cdot)$ is a pdf,
two additional redundant constraints can be defined, which will be used later to
define auxiliary variables.

\begin{const}[Volume]
  \label{const:normalization}
  Since $\psi(\cdot)$ is a pdf, it has unitary volume:
  \begin{equation*}
    \int_{-\infty}^\infty \psipdf \dx = 1.
  \end{equation*}
\end{const}

\begin{const}[Weighting]
  \label{const:weight}
  The mixture places weight $p$ in the left part of the distribution, which can
  be written as:
  \begin{equation*}
    \int_{-\infty}^\mu \psipdf \dx = p.
  \end{equation*}
\end{const}

The sampling of the new distribution $D^*$ can be performed by sampling $u \sim
\mathcal U([0,1])$ from the uniform distribution, followed by sampling from the
distribution $D_-'$ described by the non-normalized pdf $\phi_-(\cdot)$ if $u <
p$ or from $D_+'$, with non-normalized pdf $\phi_+(\cdot)$, otherwise.

Moreover, if the split parameter $\mu$ is fixed and the base distribution $D$
define an exponential family, then the new distribution $D^*$ also defines an
exponential family. An exponential family is a set of probability distributions
whose probability density functions can be expressed as
\begin{equation}
  \label{eq:exponential_family}
  f(x | \theta) = h(x) \exp\left(\eta(\theta)^T T(x) - A(\theta)\right),
\end{equation}
where $\theta$ are the parameters of the distribution and $h(x)$, $T(x)$,
$\eta(\theta)$, and $A(\theta)$ are known
functions~\citep{banerjee2005clustering}.

It is important to highlight that this result is not unexpected when using the
constrained mixture. From Constraint~\ref{const:mixture}, if the split position
$\mu$ is known, both sides behave like the underlying distribution. Therefore,
we expect the natural parameter $\eta$ to be produced by stacking the natural
parameters $\eta_-(\Theta_-)$ and $\eta_+(\Theta_+)$ for both sides. Moreover,
the sufficient statistics $T$ should be produced by stacking $T_- \I[x < \mu]$
and $T_+ \I[x \ge \mu]$, which are the statistics for each side of the
distribution.

We also note that we can not hope that the full distribution, without fixed
$\mu$, defines an exponential family too. Since the data is partitioned by
$\mu$, we cannot separate the data and parameters to create the term
$\eta(\theta)^T T(x)$ in Equation~\eqref{eq:exponential_family}.

\section{Asymmetric Distributions}
\label{sec:asymmetric_examples}
The constrained mixture defined in Section~\ref{sec:mixture} can be used to
create asymmetric versions of distributions. In this section, we will introduce
the asymmetric Laplace and normal distributions, showing that the symmetric
versions are particular cases with $p=0.5$. Later, in
Section~\ref{sec:optimization}, we will also show how to optimize the parameters
for these new distributions. To avoid cluttering, some proofs for this section
are presented in the Appendix.

To break the symmetry of these distributions, the separation parameter $\mu$ is
placed at the mode, usually also denoted by $\mu$. Therefore, the following
sections use them interchangeably, to avoid writing $\mu$ for the mixture and
$\mu'$ for the underlying distribution.
\subsection{Laplace Distribution}
\label{sec:example_laplace}
The Laplace distribution can be described by parameters $\theta = (\mu,
\lambda)$ and pdf
\begin{equation}
  \label{eq:laplace_pdf}
  \phi(x;\mu,\lambda) = \frac{\lambda}{2} \exp(-\lambda |x-\mu|).
\end{equation}
From this, we will build the asymmetric version and prove that it generalizes
the Laplace distribution.

\begin{theorem}[Asymmetric Laplace]
  \label{thm:laplace}
  Let $p \in (0,1)$, $\lambda \in (0,\infty)$, and $\mu \in \R$ be given. Then
  the pdf given by:
  \begin{equation}
    \label{eq:asym_laplace_pdf}
    \psi(x;\mu,\lambda,p) =
    \begin{dcases}
        \beta \exp(-\lambda\alpha (x-\mu)), & x \ge \mu
        \\
        \beta \exp(\lambda\alpha^{-1} (x-\mu)), & x < \mu,
      \end{dcases}
  \end{equation}
  where $\alpha = \sqrt{\frac{p}{1-p}}$ and $\beta = \frac{\lambda
  \alpha}{\alpha^2+1}$, satisfies all constraints in Section~\ref{sec:mixture}.
\end{theorem}
\begin{proof}See Appendix.\end{proof}

\begin{corollary}[Symmetric Laplace]
  Let $\lambda \in (0,\infty)$ and $\mu \in \R$ be given. Let $\phi(\cdot)$ and
  $\psi(\cdot)$ be defined as in Equations~\eqref{eq:laplace_pdf}
  and~\eqref{eq:asym_laplace_pdf}, respectively. Then the following holds:
  \begin{equation*}
    \forall x \in \R, \quad \phi(x;\mu,\lambda) = \psi(x;\mu,\lambda,0.5).
  \end{equation*}
\end{corollary}
\begin{proof}
  With $p=0.5$, we have that $\alpha = 1$ and $\beta = \lambda/2$. Using these
  values in Equation~\eqref{eq:asym_laplace_pdf}, we arrive at
  Equation~\eqref{eq:laplace_pdf}.
\end{proof}

\begin{corollary}[Asymmetric Laplace Exponential Family]
  \label{cor:laplace_exp}
  Let $\mu \in \R$ be given. Then the asymmetric Laplace pdf given by
  Equation~\eqref{eq:asym_laplace_pdf} defines an exponential family with
  functions
  \begin{subequations}
  \label{eq:laplace_exp}
  \begin{align}
    h(x) &= 1, & \quad A(\lambda, p) &= -\ln \beta,
    \\
    T(x) &= \begin{bmatrix}
      |x-\mu| \I[x \ge \mu] \\
      |x-\mu| \I[x < \mu]
    \end{bmatrix},
    & \quad
    \eta(\lambda, p) &= \begin{bmatrix}
      -\lambda \alpha \\
      -\lambda \alpha^{-1}
    \end{bmatrix}.
  \end{align}
  \end{subequations}
\end{corollary}
\begin{proof}
  Using these functions in Equation~\eqref{eq:exponential_family}, we can verify
  that it matches Equation~\eqref{eq:asym_laplace_pdf}.
\end{proof}

Figure~\ref{fig:laplace} shows the asymmetric Laplace pdf $\psi(\cdot)$ for
different combinations of $x$ and $p$ with $\mu$ fixed to $0$. It is clear that,
with $p$ getting closer to $0$, the density is more strict on negative values,
that is, they are less likely to occur. However, this also increases the
uncertainty of positive values, which exhibit a slower decay.
\begin{figure*}[t]
  \centering
  \begin{subfigure}[b]{0.50\linewidth}
    \psfrag{p}[c][c]{$p$}
    \psfrag{x}[c][c]{$x$}
    \psfrag{p(x)}[b][t]{$\psi(\cdot)$}
\includegraphics[width=\linewidth]{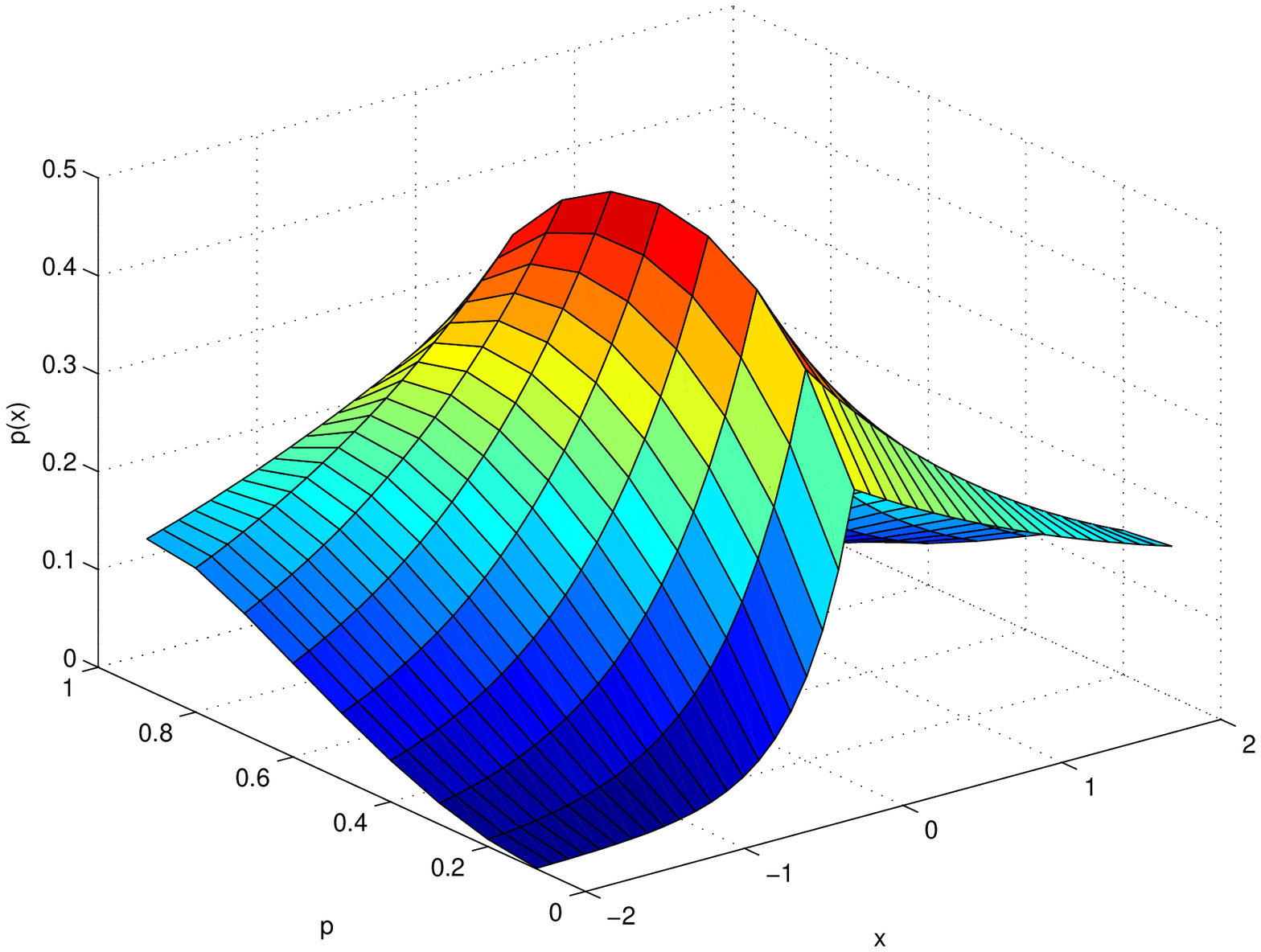}    \caption{Surface of the pdf}
  \end{subfigure}
  ~
  \begin{subfigure}[b]{0.45\linewidth}
    \psfrag{p}[b][c]{$p$}
    \psfrag{x}[t][c]{$x$}
\includegraphics[width=\linewidth]{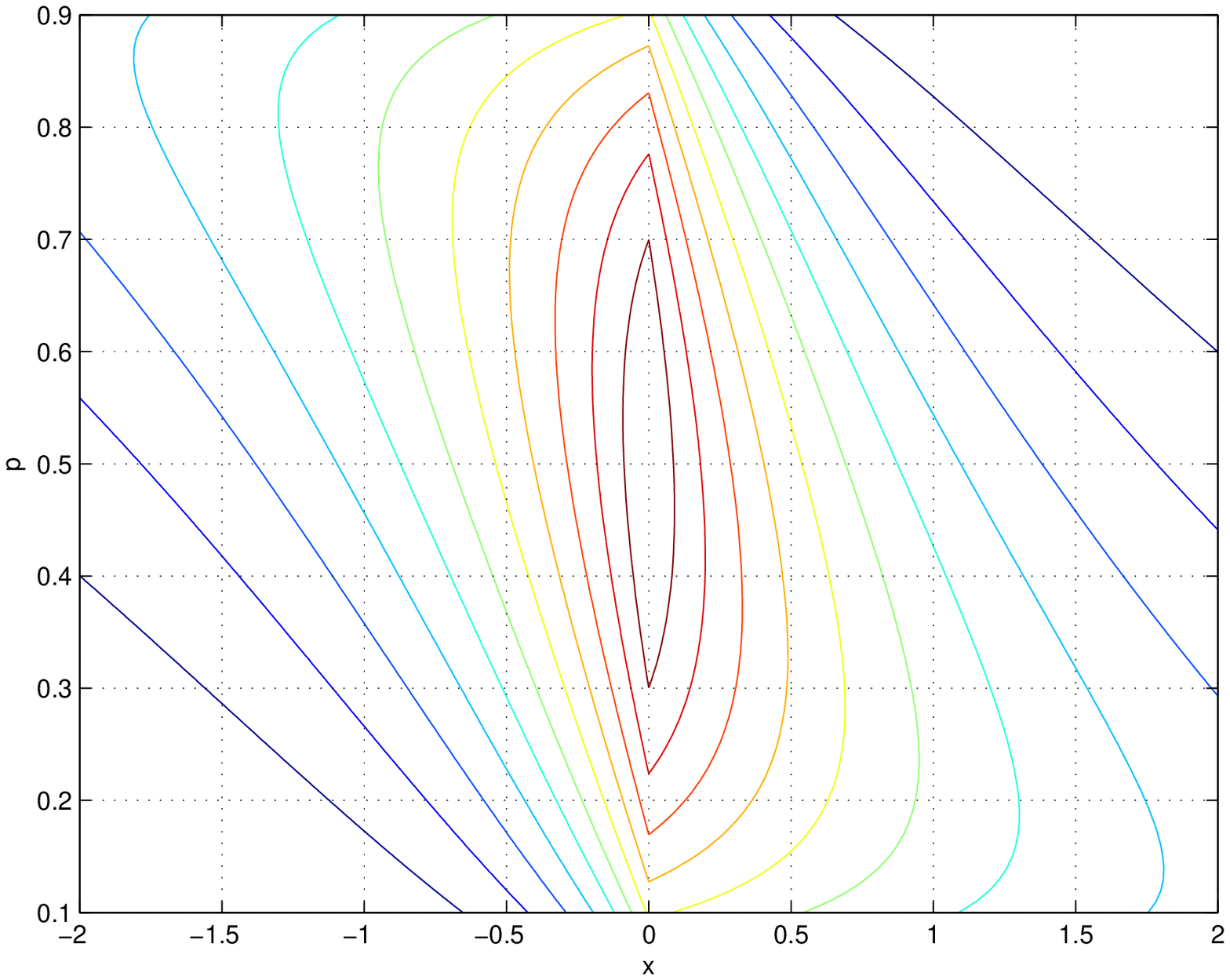}    \caption{Contours of the pdf}
  \end{subfigure}
  \caption{Asymmetric Laplace distribution for variable $p$ and $\mu = 0$. As
  $p$ gets smaller, less density is placed on negative values.}
  \label{fig:laplace}
\end{figure*}
\begin{figure*}[t]
  \centering
  \begin{subfigure}[b]{0.50\linewidth}
    \psfrag{p}[c][c]{$p$}
    \psfrag{x}[c][c]{$x$}
    \psfrag{p(x)}[b][t]{$\psi(\cdot)$}
\includegraphics[width=\linewidth]{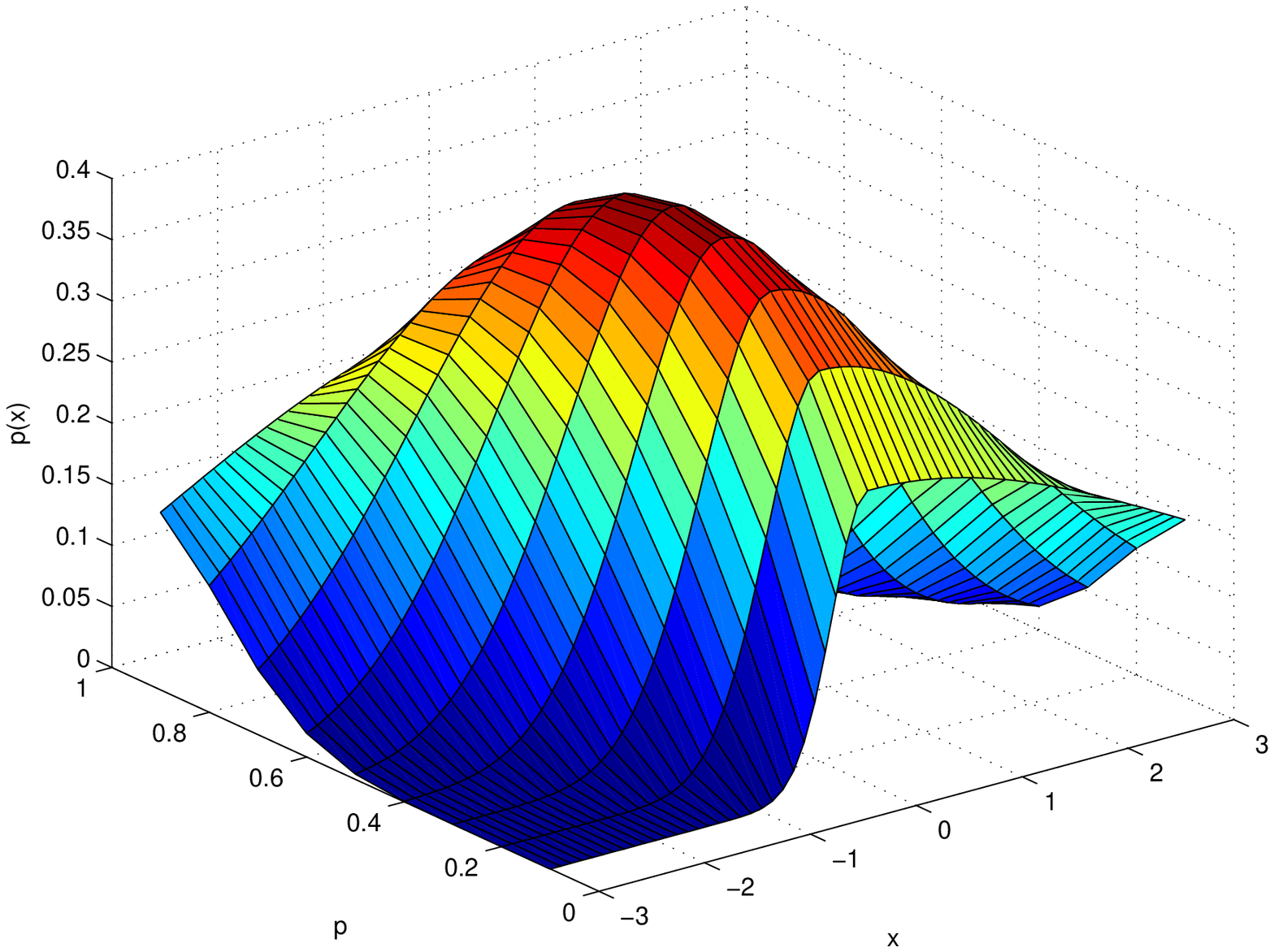}    \caption{Surface of the pdf}
  \end{subfigure}
  ~
  \begin{subfigure}[b]{0.45\linewidth}
    \psfrag{p}[b][c]{$p$}
    \psfrag{x}[t][c]{$x$}
\includegraphics[width=\linewidth]{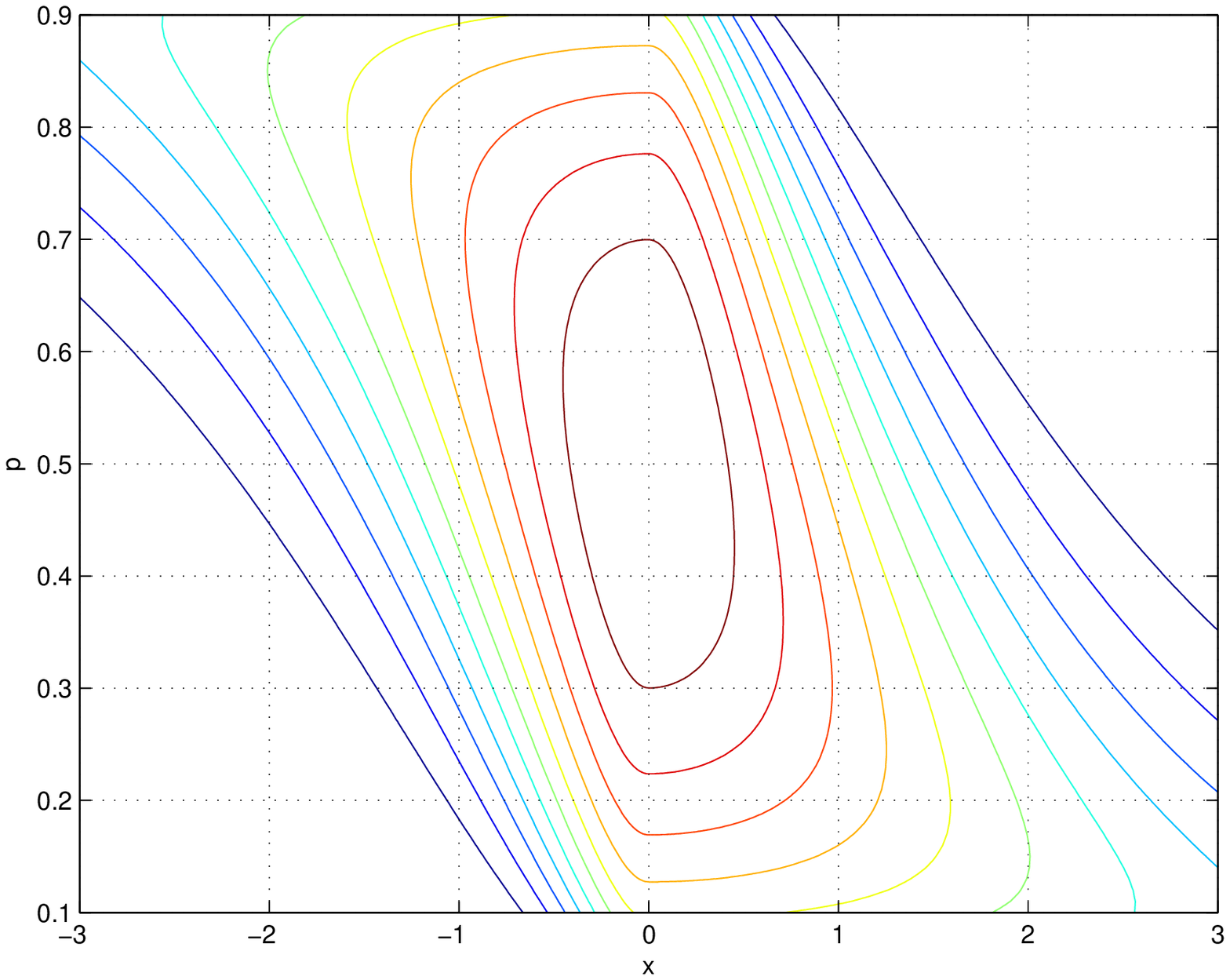}    \caption{Contours of the pdf}
  \end{subfigure}
  \caption{Asymmetric normal distribution for variable $p$ and $\mu = 0$. As
  $p$ gets smaller, less density is placed on negative values.}
  \label{fig:normal}
\end{figure*}
\subsection{Normal Distribution}
\label{sec:example_normal}
The normal distribution can be described by parameters $\theta = (\mu,\sigma)$
and pdf
\begin{equation}
  \label{eq:normal_pdf}
  \phi(x;\mu,\sigma) = \frac{1}{\sigma} \Phi\left(\frac{x-\mu}{\sigma}\right),
\end{equation}
where
\begin{equation}
  \label{eq:std_normal_pdf}
  \Phi(\xi) = \frac{1}{\sqrt{2\pi}} \exp\left(-\frac{1}{2} \xi^2\right)
\end{equation}
is the pdf of the standard normal distribution.

\begin{theorem}[Asymmetric Normal]
  \label{thm:normal}
  Let $p \in (0,1)$, $\sigma \in (0,\infty)$, and $\mu \in \R$ be given. Let
  $\Phi(\cdot)$ be defined as in Equation~\eqref{eq:std_normal_pdf}. Then the
  pdf given by:
  \begin{equation}
    \label{eq:asym_normal_pdf}
    \psi(x;\mu,\sigma,p) =
    \begin{dcases}
      \beta \Phi\left(\frac{x-\mu}{\sigma \alpha^{-1}}\right), & x \ge \mu
      \\
      \beta \Phi\left(\frac{x-\mu}{\sigma \alpha}\right), & x < \mu,
    \end{dcases}
  \end{equation}
  where $\alpha = \sqrt{\frac{p}{1-p}}$ and $\beta =
  \frac{2\alpha}{\sigma(\alpha^2+1)}$, satisfies all constraints in
  Section~\ref{sec:mixture}.
\end{theorem}
\begin{proof}See Appendix.\end{proof}

\begin{corollary}[Symmetric Normal]
  Let $\sigma \in (0,\infty)$ and $\mu \in \R$ be given. Let $\phi(\cdot)$ and
  $\psi(\cdot)$ be defined as in Equations~\eqref{eq:normal_pdf}
  and~\eqref{eq:asym_normal_pdf}, respectively. Then the following holds:
  \begin{equation*}
    \forall x \in \R, \quad \phi(x;\mu,\sigma) = \psi(x;\mu,\sigma,0.5).
  \end{equation*}
\end{corollary}
\begin{proof}
  With $p=0.5$, we have that $\alpha = 1$ and $\beta = 1/\sigma$. Using these
  values into Equation~\eqref{eq:asym_normal_pdf}, we arrive at
  Equation~\eqref{eq:normal_pdf}.
\end{proof}

\begin{corollary}[Asymmetric Normal Exponential Family]
  \label{cor:normal_exp}
  Let $\mu \in \R$ be given. Then the asymmetric normal pdf given by
  Equation~\eqref{eq:asym_normal_pdf} defines an exponential family with
  functions
  \begin{subequations}
  \label{eq:normal_exp}
  \begin{align}
    h(x) &= \frac{1}{\sqrt{2\pi}}, & \quad A(\sigma, p) &= -\ln \beta,
    \\
    T(x) &= \begin{bmatrix}
      (x-\mu)^2 \I[x \ge \mu] \\
      (x-\mu)^2 \I[x < \mu]
    \end{bmatrix},
    & \quad
    \eta(\sigma, p) &= \begin{bmatrix}
    \displaystyle
      -\frac{1}{2 \sigma^2 \alpha^{-2}} \\
    \displaystyle
      -\frac{1}{2 \sigma^2 \alpha^2}
    \end{bmatrix}.
  \end{align}
  \end{subequations}
\end{corollary}
\begin{proof}
  Using these functions in Equation~\eqref{eq:exponential_family}, we can verify
  that it matches Equation~\eqref{eq:asym_normal_pdf}.
\end{proof}

Figure~\ref{fig:normal} shows the asymmetric normal pdf $\psi(\cdot)$ for
combinations of $x$ and $p$ with $\mu$ fixed to $0$. Just like the asymmetric
Laplace distribution, $p$ values closer to $0$ are more strict on negative
values, making the distribution more conservative on these cases.

\section{Parameter Optimization}
\label{sec:optimization}
Once defined the new distributions, we are interested in adjusting their
parameters to fit some data set. However, mixture models involve latent
variables, such as the indicator of to which class a given sample belongs in
standard mixture or the current state in hidden Markov models. For an asymmetric
distribution, the indicator is given deterministically from $\mu$, since we just
have to identify if the observed value is larger or smaller than the parameter
$\mu$. This parameter, in turn, depends on the weight $p$, which specifies how
much probability to give to each side of $\mu$.

This dependency between parameters makes the analysis and optimization process
more complicated and, in our development, we were not able to find a solution to
simultaneously optimize $\theta$, $\mu$, and $p$ at the same time while
providing guarantees. However, if we fix either $\mu$ or $p$, then we are able
to find formulations to optimize the others.

Let $S = \{s_i\}, i \in \{1,2,\ldots,N\},s_i \in \R$, be a set of samples. Using
Constraint~\ref{const:mixture}, the parameter's log-likelihood can be written
as:
\begin{subequations}
\begin{gather}
  \label{eq:general likelihood}
  \ln \mathcal L(\mu,\theta,p | S) = -|S| \ln Z + \ln \mathcal L_p + \ln
  \mathcal L_\phi
  \\
  \label{eq:common_likelihood}
  \ln \mathcal L_p = |S_-| \ln p + |S_+| \ln (1-p)
  \\
  \label{eq:distribution_likelihood}
  \ln \mathcal L_\phi =
  \sum_{s_i \in S_-} \ln \phi_-(s_i;\Theta_-(\cdot)) +
  \sum_{s_i \in S_+} \ln \phi_+(s_i;\Theta_+(\cdot)),
\end{gather}
\end{subequations}
where $S_- = \{s_i \in S | s_i < \mu\}$ and $S_+ = \{s_i \in S | s_i \ge
\mu\}$.

If we consider the parameter $p$ fixed, then the maximum likelihood problem
for both distributions has known optima, and they have closed-form expressions,
as we will show in Sections~\ref{sec:optimization:laplace}
and~\ref{sec:optimization:normal}. Since $p$ is only one value, it can be
optimized numerically, as described in Section~\ref{sec:optimization:asymmetry}.

Alternatively, since both distributions were shown to define the exponential
families in Section~\ref{sec:asymmetric_examples} when $\mu$ is fixed and
exponential families have conjugate priors~\citep{barndorff2014information},
then the new distributions must have conjugate priors. Moreover, the conjugate
priors probability density function can be written as
\begin{equation}
  \label{eq:prior}
  p(\eta | \chi, \nu) = f(\chi, \nu) \exp(\eta^T \chi - \nu A(\eta)),
\end{equation}
where $\eta$ and $A(\eta)$ are the natural parameters and a function of them. In
this section, we will also find the priors and show that their structure is
sound. With these priors, one could compute the posterior distribution over the
parameters~\citep{barndorff2014information,bishop2006pattern} or use the new
distributions as part of a more complex model with intractable closed-form,
using an approach such as variational inference~\citep{blei2003latent} or Gibbs
sampling~\citep{geman1984stochastic}, since the best approximating posterior is
the conjugate prior.

Therefore, we provide two methods for optimizing the parameters, one where the
partition weight $p$ is defined and we compute the maximum likelihood, and one
where the partitions themselves are defined through a fixed $\mu$ and we can
compute the full posterior on the parameters. It is important to highlight that,
since the symmetric distributions are particular cases of the asymmetric ones,
their likelihoods can not be higher than the asymmetric likelihoods for the same
set data set. All proofs for this section are presented in the Appendix.
\subsection{Laplace Distribution}
\label{sec:optimization:laplace}
Using the functions defined in the constrained mixture and in the proof of
Theorem~\ref{thm:laplace}, the distribution-specific likelihood, given by
Equation~\eqref{eq:distribution_likelihood} can be written as:
\begin{equation}
  \label{eq:laplace_likelihood}
  \ln \mathcal L_\phi = |S| \ln \lambda + (|S_+|-|S_-|) \ln \alpha
  +\lambda\left(\alpha^{-1} \sum_{s_i \in S_-} (s_i-\mu) -
  \alpha \sum_{s_i \in S_+} (s_i-\mu) \right).
\end{equation}

Using $\mathcal L_p$ from Equation~\eqref{eq:common_likelihood} and the second
term in the previous equation, we can verify that
\begin{subequations}
\label{eq:entropy}
\begin{align}
  &|S_-|\ln p + |S_+| \ln q + (|S_+|-|S_-|)\ln \alpha
  \\
  &=\frac{|S_-|}{2} (\ln p + \ln q) + \frac{|S_+|}{2} (\ln p + \ln q)
  = \frac{|S|}{2} (\ln p + \ln q)
  \\
  \label{eq:kl_divergence}
  &= -|S| H(Be(0.5)) - |S| D_{KL}(Be(0.5)||Be(p)),
\end{align}
\end{subequations}
where $q = 1-p$, $Be(p)$ is the Bernoulli distribution, $H(\cdot)$ is the
entropy, and $D_{KL}(\cdot)$ is the Kullback-Leibler
divergence~\citep{kullback1951information}. Therefore, only the first and third
terms in Equation~\eqref{eq:laplace_likelihood} change with $\mu$ and $\lambda$.

Moreover, the likelihood term that depends only on $p$ decreases as $p$ moves
away from the symmetric version $p=0.5$. This can be viewed as an implicit
regularization of the asymmetry, since it comes directly from the distributions
defined in Section~\ref{sec:asymmetric_examples} and reduces the likelihood as
the asymmetry increases. Therefore, the distribution only becomes more
asymmetric whenever the likelihood gain in data fitting is higher than the loss
of becoming more asymmetric.

\begin{theorem}[Asymmetric Laplace Optimality]
  \label{thm:laplace_opt}
  Let $p \in (0,1)$ and $S = \{s_i\}$, $i \in \{1,2,\ldots,N\}, s_i \in \R$, be
  given. Let the pdf of the asymmetric Laplace distribution be given by
  Equation~\eqref{eq:asym_laplace_pdf}. Then the likelihood has an optimum where
  the partition $\mu^*$ is given by the weighted median, with samples in $S_-$
  and $S_+$ weighted by $\alpha^{-1}$ and $\alpha$, respectively, and
  \begin{equation*}
    \lambda^* = \frac{|S|}{\alpha \sum_{s_i \in S_+} (s_i-\mu) -
    \alpha^{-1} \sum_{s_i \in S_-} (s_i-\mu)},
  \end{equation*}
  where $S_- = \{s_i \in S | s_i < \mu^*\}$, $S_+ = \{s_i \in S | s_i >
  \mu^*\}$, and $\alpha = \sqrt{\frac{p}{1-p}}$.

  Furthermore, let $\mu^*_1$ and $\mu^*_2$, $\mu^*_1 < \mu^*_2$, be optimal
  partitions. Then there is no $s_i \in S$ such that $\mu^*_1 < s_i < \mu^*_2$,
  that is, all optimal partitions induce the same sets $S_-$ and $S_+$.
\end{theorem}
\begin{proof}See Appendix.\end{proof}

It is important to highlight that we have to look at all possible partitions of
$S$, compute their optimal $\mu^*$ given by the median, and check whether it
induces the same partition. Since all optima induce the same partition, only one
such median induce the partition used to create it, with the other values
falling outside the required interval $\max S_- < \mu^* < \min S_+$.

Alternatively, if we consider $\mu$ fixed instead of $p$, we have shown in
Section~\ref{sec:example_laplace} that the asymmetric Laplace defines an
exponential family, which means that it has a conjugate prior given by
Equation~\eqref{eq:prior}, where $\eta$ and $A(\eta)$ are defined in
Equation~\eqref{eq:laplace_exp}.

\begin{theorem}[Asymmetric Laplace Conjugate Prior]
  \label{thm:laplace_prior}
  Let the asymmetric Laplace distribution be given by
  Equation~\eqref{eq:asym_laplace_pdf}, with exponential family functions given
  by Equation~\eqref{eq:laplace_exp}. Then its conjugate prior probability
  density function is given by
  \begin{equation*}
    f(p, \lambda; \nu, \chi) =
    G(\lambda \alpha;\nu, \chi_1)
    G(\lambda \alpha^{-1};\nu, \chi_2)
    B(p; \nu'),
  \end{equation*}
  where
  \begin{equation*}
    G(\Lambda; \alpha, \beta) =
    \frac{\beta^\alpha}{\Gamma(\alpha)}
    \Lambda^{\alpha-1} \exp\left(-\Lambda \beta\right)
  \end{equation*}
  is the gamma distribution, $\Gamma(\cdot)$ is the gamma function,
  \begin{equation*}
    B(p; \alpha) = \frac{1}{B(\alpha, \alpha)} p^{\alpha-1} (1-p)^{\alpha-1}
  \end{equation*}
  is the symmetric beta distribution, $B(\cdot)$ is the beta function, and
  $\alpha = \sqrt{\frac{p}{1-p}}$.
\end{theorem}
\begin{proof}See Appendix.\end{proof}

Since the prior for the Laplace distribution, in the format written in
Equation~\eqref{eq:laplace_pdf}, is the gamma distribution, and the prior for
$p$, which can be seen as a parameter in a Bernoulli distribution deciding in
which side of $\mu$ the data will fall, is a beta distribution, it is reasonable
to expect that the asymmetric Laplace prior has one gamma distribution for each
side and one beta distribution for the deciding parameters, with their
hyperparameters linked in a way that the final parameters always satisfy the
conditions for a constrained mixture.
\subsection{Normal Distribution}
\label{sec:optimization:normal}
Using the functions defined in the constrained mixture and in the proof of
Theorem~\ref{thm:normal}, the distribution-specific likelihood, given by
Equation~\eqref{eq:distribution_likelihood} can be written as:
\begin{equation}
  \label{eq:normal_likelihood}
    \ln \mathcal L_\phi = C - |S| \ln \sigma + (|S_+|-|S_-|) \ln \alpha
    -\frac{\sum_{s_i \in S_-} {(s_i-\mu)}^2}{2\sigma^2\alpha^2} -
    \frac{\alpha^2 \sum_{s_i \in S_+} {(s_i-\mu)}^2}{2\sigma^2},
\end{equation}
where $C$ is a constant.

Similarly to Equation~\eqref{eq:entropy}, we can show that the term associated
with $\ln \alpha$ does not depend on the partition, once we consider $\mathcal
L_p$. Therefore, only the other terms are used in the optimization.

\begin{theorem}[Asymmetric Normal Optimality]
  \label{thm:normal_opt}
  Let $p \in (0,1)$ and $S = \{s_i\}, i \in \{1,2,\ldots,N\}, s_i \in \R$, be
  given. Let the pdf of the asymmetric normal distribution be given by
  Equation~\eqref{eq:asym_normal_pdf}. Then the likelihood has a single optimum,
  where the optimal partition is given by
  \begin{equation*}
    \mu^* = \frac{\alpha^{-2} \sum_{s_i \in S_-} s_i + \alpha^2 \sum_{s_i \in
    S_+} s_i}{\alpha^{-2} |S_-| + \alpha^2 |S_+|}
  \end{equation*}
  and
  \begin{equation*}
    {\sigma^*}^2 = \frac{\alpha^{-2} \sum_{s_i \in S_-}
    {(s_i-\mu)}^2 + \alpha^2 \sum_{s_i \in S_+} {(s_i-\mu)}^2}{|S|},
  \end{equation*}
  where $S_- = \{s_i \in S | s_i < \mu^*\}$, $S_+ = \{s_i \in S | s_i >
  \mu^*\}$, and $\alpha = \sqrt{\frac{p}{1-p}}$.
\end{theorem}
\begin{proof}See Appendix.\end{proof}

Similarly to the asymmetric Laplace, we have to look at all partitions and check
whether the optimal $\mu^*$ is valid for that partition.

Also similarly to the asymmetry Laplace, if we consider $\mu$ fixed instead of
$p$, we have shown in Section~\ref{sec:example_normal} that the asymmetric
normal defines an exponential family, which means that it has a conjugate prior
given by Equation~\eqref{eq:prior}, where $\eta$ and $A(\eta)$ are defined in
Equation~\eqref{eq:normal_exp}.

\begin{theorem}[Asymmetric Normal Conjugate Prior]
  \label{thm:normal_prior}
  Let the asymmetric normal distribution be given by
  Equation~\eqref{eq:asym_normal_pdf}, with exponential family functions given
  by Equation~\eqref{eq:normal_exp}. Then its conjugate prior probability
  density function is given by
  \begin{equation*}
    f(p, \sigma; \nu, \chi) =
    Ig(\sigma^2 \alpha^2;\nu_2, \chi_2)
    Ig(\sigma^2 \alpha^{-2};\nu_2, \chi_1)
    B(p; \nu_1),
  \end{equation*}
  where
  \begin{equation*}
    Ig(\Sigma; \alpha, \beta) = \frac{\beta^\alpha}{\Gamma(\alpha)}
    \Sigma^{-\alpha-1} \exp\left(-\frac{\beta}{\Sigma}\right)
  \end{equation*}
  is the inverse gamma distribution, $\Gamma(\cdot)$ is the gamma function,
  \begin{equation*}
    B(p; \alpha) = \frac{1}{B(\alpha, \alpha)} p^{\alpha-1} (1-p)^{\alpha-1}
  \end{equation*}
  is the symmetric beta distribution, $B(\cdot)$ is the beta function,
 $\alpha = \sqrt{\frac{p}{1-p}}$, $\nu_1 = 1+\nu/2$, and $\nu_2 = \nu/4-1$.
\end{theorem}
\begin{proof}See Appendix.\end{proof}

Again, just like the asymmetric Laplace, the prior is in agreement with what is
expected, since the prior for a variance is the inverse gamma distribution and
the prior for $p$ is a beta distribution.
\subsection{Asymmetry Parameter}
\label{sec:optimization:asymmetry}
Sections~\ref{sec:optimization:laplace} and~\ref{sec:optimization:normal} showed
how $\mu$ and $\theta$ can be optimized in a closed form to maximize the
likelihood for a fixed $p$. Since $p$ is a single value, it can be optimized
efficiently with a hill-climbing algorithm.

Given a value of $p$, the log-likelihood can be written as in
Equation~\eqref{eq:general likelihood}. Let
\begin{equation*}
  L(p) =
  \begin{dcases}
    \ln \mathcal L(\mu^*, \theta^*, p | S), & p \in (0,1)
    \\
    -\infty, & \text{otherwise,}
  \end{dcases}
\end{equation*}
where $\mu^*$ and $\theta^*$ are the optimal values for a given $p$.
Let the initial estimate of $p$ be $p^{(0)} = 0.5$, the initial step
$\eta^{(0)} > 0$, the tolerance $\epsilon > 0$ and the adjustment $1 > \gamma >
0$ be given. Then the hill-climbing algorithm works as follows:
\begin{enumerate}
  \item Initialize $i = 0$ and $p^{(0)} = 0.5$.
  \item Let $p^{(i)}_- = p^{(i)} - \eta$ and $p^{(i)}_+ = p^{(i)} + \eta$.
  \item If $\eta^{(i)} < \epsilon$, stop.
  \item Let $L^{(i)} = L(p^{(i)})$, $L^{(i)}_- = L(p^{(i)}_-)$, and $L^{(i)}_+ =
    L(p^{(i)}_+)$.
  \item If $L^{(i)}_+ \ge L^{(i)}$, then $p^{(i+1)} = p^{(i)}_+$, $p^{(i+1)}_- =
    p^{(i)}$, $p^{(i+1)}_+ = p^{(i)}_+ + \eta^{(i)}$, and $\eta^{(i+1)} =
    \eta^{(i)}$. Go to step 4 with $i = i+1$.
  \item If $L^{(i)}_- \ge L^{(i)}$, then $p^{(i+1)} = p^{(i)}_-$, $p^{(i+1)}_+ =
    p^{(i)}$, $p^{(i+1)}_- = p^{(i)}_- - \eta^{(i)}$, and $\eta^{(i+1)} =
    \eta^{(i)}$. Go to step 4 with $i = i+1$.
  \item Let $\eta^{(i+1)} = \eta^{(i)} \gamma$. Go to step 2 with $i = i+1$.
\end{enumerate}

This simple algorithm keeps the best estimate of $p$ at $p^{(i)}$ and compares
it with its $\eta^{(i)}$ neighbors, moving to the direction that maximizes the
likelihood. If the central estimate is the better, the step is reduced and the
process is repeated until convergence.

If the asymmetric distribution is part of a mixture, as in the example in
Section~\ref{sec:applications:hmm}, then we must take certain precautions to
avoid prematurely choosing a value of $p$. We have found that fixing the value
of $p$ to $0.5$, such that the distribution behaves like its symmetric version,
until convergence of the likelihood, and then performing the hill-climbing every
time a maximum was being fit for the asymmetric distribution, thus allowing $p$
to change, provided very good results and was able to avoid poor minima due to
premature compromise of the value of $p$. Therefore, we first solve the
symmetric problem until convergence, which should have less local minima due to
less flexibility, then use its estimated parameters as initial conditions for
the asymmetric problem, guaranteeing that the likelihood can only increase.

\section{Applications of the Proposed Asymmetric Distributions}
\label{sec:applications}
To demonstrate the characteristics of the new distributions, we propose two
applications to compare the symmetric and asymmetric versions: one toy example
to understand the fundamentals and one real world example to explore deeper
characteristics of the distribution. Since the normal distribution is frequently
used, both applications will focus on it.

A standard basic problem in machine learning is performing a linear regression
to fit some data. Therefore the toy problem is composed of a linear regression,
where the noise can be asymmetric. In this case, we will show that the
asymmetric normal is able to consistently adapt to this asymmetry when it is
present, providing higher likelihoods.

We note that there are approaches that use asymmetric noise models, such as the
log-gamma distribution~\citep{bianco2005robust}, to perform the linear
regression, but these other distributions may be unknown to the user and may be
difficult to interpret. However, the normal distribution is very common and most
people are familiar with it, which makes the new asymmetric normal distribution
a good candidate for noise model, since each side of the partition can be
interpreted as a normal distribution.

The real world problem is given by learning a time-series using a hidden Markov
model, where the emission distributions have now the flexibility of being
asymmetric. We will show that this extra flexibility not only increases the
likelihood, but may be able to reduce the entropy of the model.
\subsection{Asymmetric Linear Regression}
\label{sec:applications:toy}
The standard linear regression problem is defined by finding a parameter vector
$\beta \in \R^M$ such that the relationship between an input $x \in \R^D$ and an
output $y \in \R$ can be described by
\begin{equation*}
  y = \beta^T \phi(x) + \epsilon, \quad
  \epsilon \sim N(0, \sigma^2),
\end{equation*}
where $\phi(x) \colon \R^D \to \R^M$ is a function that computes features of the
input and $N(\mu,\sigma^2)$ is a normal distribution with mean $\mu$ and
variance $\sigma^2$~\citep{bishop2006pattern}. One of the basic choices of
$\phi(x)$ is the linear function, given by $\phi(x) = [x, 1]$, such that $\beta$
gives the slop and offset of a straight line.

With the asymmetric normal distribution, introduced in
Section~\ref{sec:example_normal}, it is possible to generalize this model to
include asymmetric noise, such that the relationship between input and output
becomes
\begin{equation*}
  y = \beta^T \phi(x) + \epsilon, \quad
  \epsilon \sim N_a(0, \sigma^2, p),
\end{equation*}
where $N_a(\mu, \sigma^2, p)$ is an asymmetric normal with partition $\mu$,
underlying variance $\sigma^2$, and weighting $p$.

\begin{figure}[t]
  \centering
  \psfrag{x}[c][c]{\small $x$}
  \psfrag{y}[c][c]{\small $y$}
\includegraphics[width=0.48\linewidth]{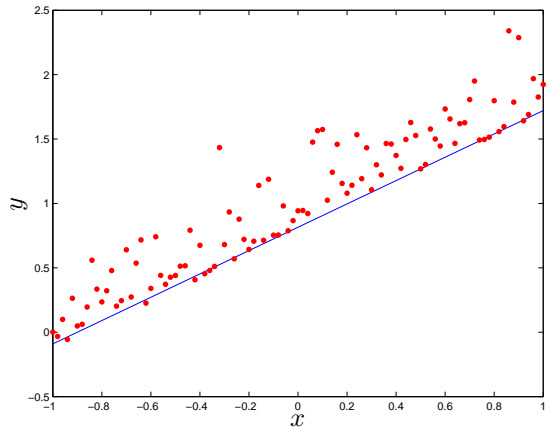}  \caption{Linear example with asymmetric noise with $p = 0.1$.}
  \label{fig:toy:example}
\end{figure}

Figure~\ref{fig:toy:example} shows an example of using the asymmetric normal
with $\sigma = 0.1$ and $p = 0.1$. The straight line is the noise-less
relationship and the dots are the noised samples obtained. Since $p < 0.5$, the
distribution creates less points with negative measurement errors and makes the
positive errors larger. From this image, it is clear that a standard normal is
not able to fit well the noise, since the region with high concentration of
points is close to the line, but it is concentrated on one side of the mean
noise.

We performed $100$ simulations for each value of $p \in \{0.1,0.2,\ldots,0.9\}$,
where in each run the values of $\beta$ were sampled uniformly in the interval
$[-1,1]$ and the underlying standard deviation $\sigma$ was set to $0.1$. The
inputs, which were shared by all simulations, were given by $101$ equidistant
points between $-1$ and $1$.

\begin{figure*}[t]
  \centering
  \begin{subfigure}[b]{0.48\linewidth}
    \psfrag{Symmetric}[c][c]{\small Symmetric}
    \psfrag{Asymmetric}[c][c]{\small Asymmetric}
\includegraphics[width=\linewidth]{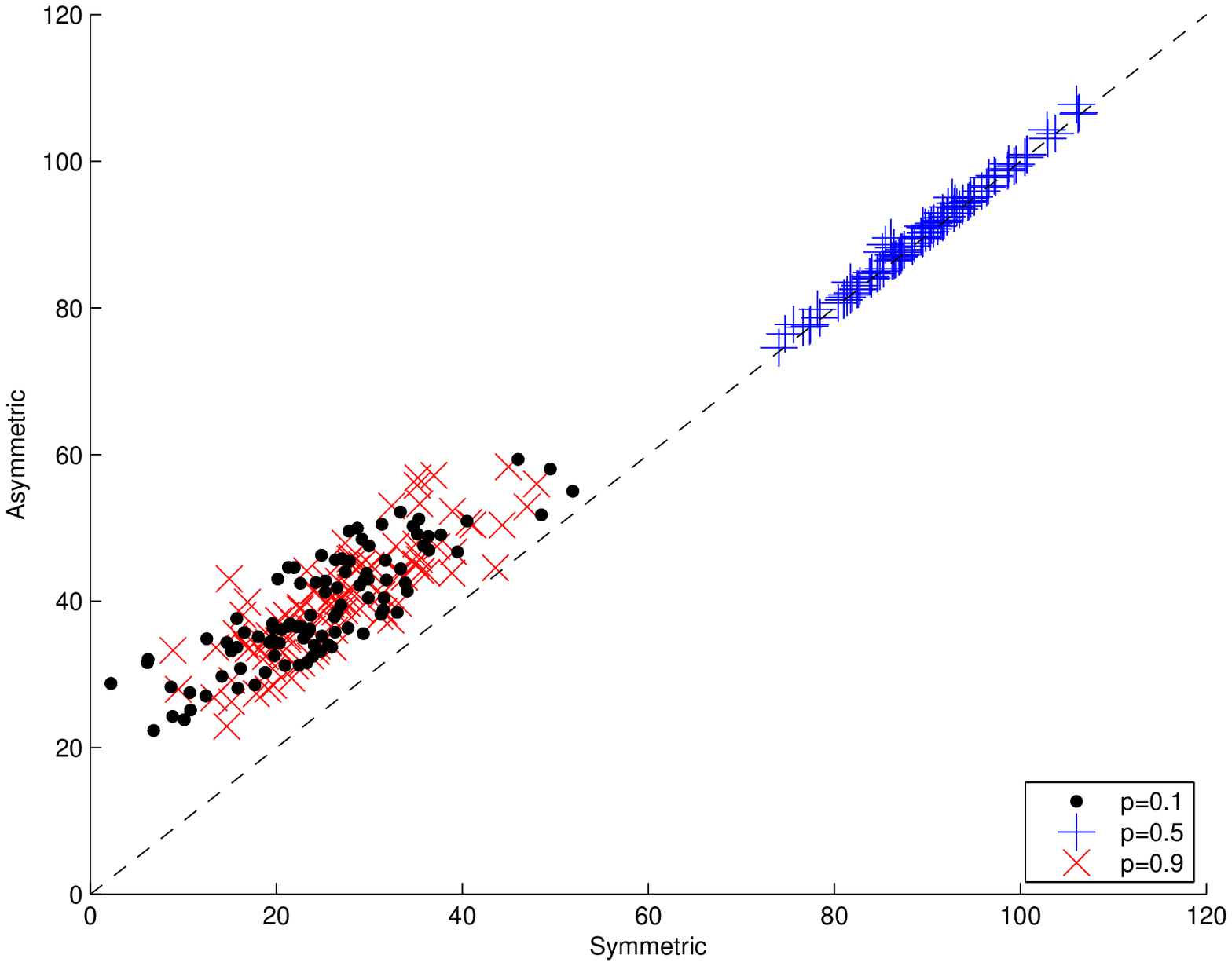}    \caption{Log-likelihood}
    \label{fig:toy:ll}
  \end{subfigure}
  ~
  \begin{subfigure}[b]{0.48\linewidth}
    \psfrag{p}[c][c]{\small $p$}
    \psfrag{ppred}[b][c]{\small $\hat p$}
\includegraphics[width=\linewidth]{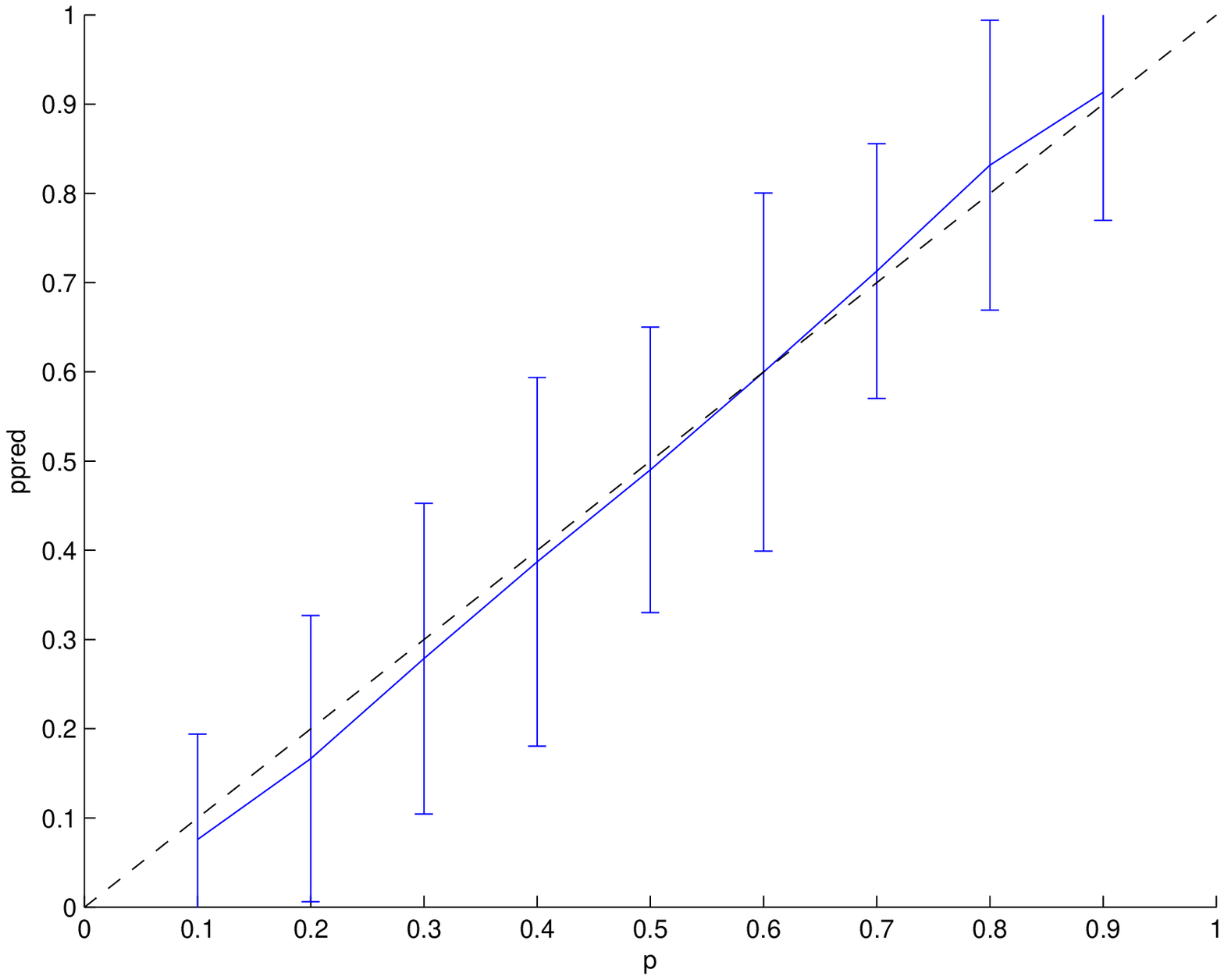}    \caption{Prediction of the asymmetry.}
    \label{fig:toy:p_pred}
  \end{subfigure}
  \caption{Results of learning a linear regression model to data with asymmetric
  noise.}
\end{figure*}

Figure~\ref{fig:toy:ll} shows the resulting likelihood of the fitted model,
where the dashed line represents equal likelihood. When $p=0.5$, both models
exhibit similar likelihoods, as we expected since this case describes the
symmetric normal distribution. Furthermore, since the symmetric normal is a
particular case of the asymmetric one, its likelihood can not be higher than the
likelihood of the asymmetric normal. In fact, the asymmetric normal has higher
likelihood in all simulations performed. However, when we set $p = 0.1$ or
$0.9$, both symmetric and asymmetric models have lower likelihood, with the
asymmetric one fitting better, as expected. The decrease in likelihood for the
asymmetric model can be explained in part by Equation~\eqref{eq:kl_divergence},
where we have shown that the model loses likelihood by making $p$ more distant
from $0.5$, while the decrease for the symmetric normal is due to incorrect
noise modelling.

It is important to highlight that, just like the terms $\ln \lambda$ in
Equation~\eqref{eq:laplace_likelihood} and $-\ln \sigma$ in
Equation~\eqref{eq:normal_likelihood} which prevents the error terms in the same
equations to have almost no weight, the cost in
Equation~\eqref{eq:kl_divergence} can be viewed as an implicit regularization
that prevents one side of the partition to have no weight, and this
regularization is inherent to the distributions defined in
Equations~\eqref{eq:asym_laplace_pdf} and~\eqref{eq:asym_normal_pdf} and is not
artificially imposed.

Moreover, there is a similarity between the resulting likelihoods for $p=0.1$
and $p=0.9$. This is expected, since there is a similarity between the two, with
$p=0.1$ favoring positive noises as much as $p=0.9$ favors negative ones.

Figure~\ref{fig:toy:p_pred} shows the correct and predicted values of $p$, again
with the dashed line representing the identity function, where predicted values
$\hat p$ are represented by their mean and $95\%$ confidence interval. The mean
prediction is clearly close to the true value, and the large variation of fitted
weights $\hat p$ is due to the small number of samples, since the model is more
flexible. However, when comparing the likelihood values for $p=0.5$ in
Figure~\ref{fig:toy:ll}, we see that the large spread of predicted values, from
$\hat p = 0.35$ to $\hat p = 0.65$ approximately, does not interfere in the
likelihood, as the value is similar to the normal that has $p = 0.5$.

Although it might seem that the results in Figure~\ref{fig:toy:p_pred} are not
the maximum likelihood estimates, since they may be far from the real parameter
$p$ used to create the noise, we remind the reader that they may differ for a
finite number of samples, just like any other estimate. For instance, for $M$
samples $x_i \sim N(\mu, \sigma^2)$ drawn from the normal distribution, the
maximum likelihood estimate for the mean is given by $\hat \mu = \sum_{i=1}^M
x_i/M$, but this estimate depends on the value of the specific $x_i$ sampled. If
we consider the uncertainty on $x_i$, it can be shown that the estimate is given
by $\hat \mu \sim N(\mu, \sigma^2/M)$ \citep{krishnamoorthy2006handbook}, which
specifies a random variable that only converges to the real value $\mu$ as $M
\to \infty$. Therefore, for finite number of samples, the parameter $\hat p$ may
differ from $p$ and still be a maximum likelihood estimate.

Therefore, we have shown that the asymmetric normal noise model is able to fit
as well as the symmetric normal when the noise is indeed symmetric, and
outperforms it when there is asymmetry in the noise. This motivates the use of
the asymmetric normal distribution as a generalization of the normal
distribution, thus being able to adapt to the observed noise asymmetry.
\subsection{Hidden Markov Model with Asymmetric Emissions}
\label{sec:applications:hmm}
While the creation of more flexible distributions by introducing the asymmetry
is in itself interesting, with the possibility of fitting different data while
keeping the interpretability, its use may also provide additional insights of
practical relevance. To
illustrate the application of the new distributions, we will use a hidden Markov
model (HMM) to fit a time-series.

A HMM with $K$ states is defined by the initial distribution on the states
$\pi$, the transition matrix between states $T$, and the parameters for each
distribution associated with each state $\theta_1,\ldots,\theta_K$. For the
normal distribution, $\theta_i$ is given by $\mu_i$ and $\sigma_i$, while for
the asymmetric version, $p_i$ is also included. In this application, we will
build two HMM, one with only symmetric and one with only asymmetric normal
distributions.

To improve the initial estimates for the HMM, we first fit the data using a
mixture model with weights $w$ and with the same parameters for the
emission distributions. Once the expectation maximization algorithm runs for 100
iterations, we set $\pi = w$ and $T = w 1_{1 \times K}$, such that every sample
has the same prior probability over the emission distributions. Additionally,
for the asymmetric version, we first fit the samples, both for the mixture and
the HMM, using the method described in Section~\ref{sec:optimization:asymmetry}.

The data used was the Dow Jones Industrial Average index (DJI) from its first
quotation, on Jan 29, 1985, to its last quotation of 2014, on Dec 31, 2014, with
the prices adjusted for dividends and splits, where we consider that its value
follows the lognormal distribution, as usual in the economics
field~\citep{aitchison1957lognormal}. Each sample is composed of the return over
investment's (ROI) logarithm for consecutive days, that is, the sample is given
by $s_i = \log(v_{i+1}/v_i)$, where $v_i$ is the quotation in the $i$-th day. If
either day of a pair does not have a quotation, what happens if one of them is
on a weekend for instance, then that sample is considered missing. Therefore,
the HMM has one state for each day between those dates.

The main motivation of using this kind of problem is that the hypothesis of
symmetry implied by the normal distribution may not reflect the reality. It is
well known that stock markets can have periods of very high or low return, which
sometimes characterize bull or bear markets~\citep{edwards2013technical}.
Therefore, we expect to see improvements by introducing an emission distribution
that is able to exhibit such asymmetric behavior.

Table~\ref{tab:application:likelihood} shows the final log-likelihood for the
samples with different number of possible states $K$. As expected, using the
asymmetric distribution provides greater likelihood due to its additional
flexibility. Moreover, increasing the number of states also increases the
difference in likelihood. Since the number of states in which the HMMs differ
the most is given by $K=5$, the subsequent analysis will consider only this
case.

\begin{table}[b]
  \centering
  \caption{Parameters' log-likelihood}
  \label{tab:application:likelihood}
  \begin{tabular}{|c|c|c|}
    \hline
    K & Symmetric & Asymmetric
    \\
    \hline
    2 & 19310.47 & 19310.91
    \\
    3 & 19480.27 & 19481.32
    \\
    4 & 19509.24 & 19514.80
    \\
    5 & 19519.82 & 19538.44
    \\
    \hline
  \end{tabular}
\end{table}

\begin{table}[tb]
  \centering
  \caption{Transition entropy (bits)}
  \label{tab:application:entropy}
  \begin{tabular}{|c|c|c|}
    \hline
    Source state & Symmetric & Asymmetric
    \\
    \hline
    1 & 0.1256 & 0.0847
    \\
    2 & 1.4972 & 0.4564
    \\
    3 & 0.4320 & 0.2108
    \\
    4 & 0.1815 & 0.1836
    \\
    5 & 1.1986 & 0.3565
    \\
    \hline
  \end{tabular}
\end{table}

\begin{figure*}[tb]
  \centering
  \begin{subfigure}[b]{0.48\linewidth}
    \psfrag{C1}[c][l]{\scriptsize C1}
    \psfrag{C2}[c][l]{\scriptsize C2}
    \psfrag{C3}[c][r]{\scriptsize C3}
    \psfrag{C4}[c][r]{\scriptsize C4}
    \psfrag{C5}[c][l]{\scriptsize C5}
    \psfrag{x}[c][c]{\small $x$}
    \psfrag{p(x)}[b][c]{\small $p(x)$}
\includegraphics[width=\linewidth]{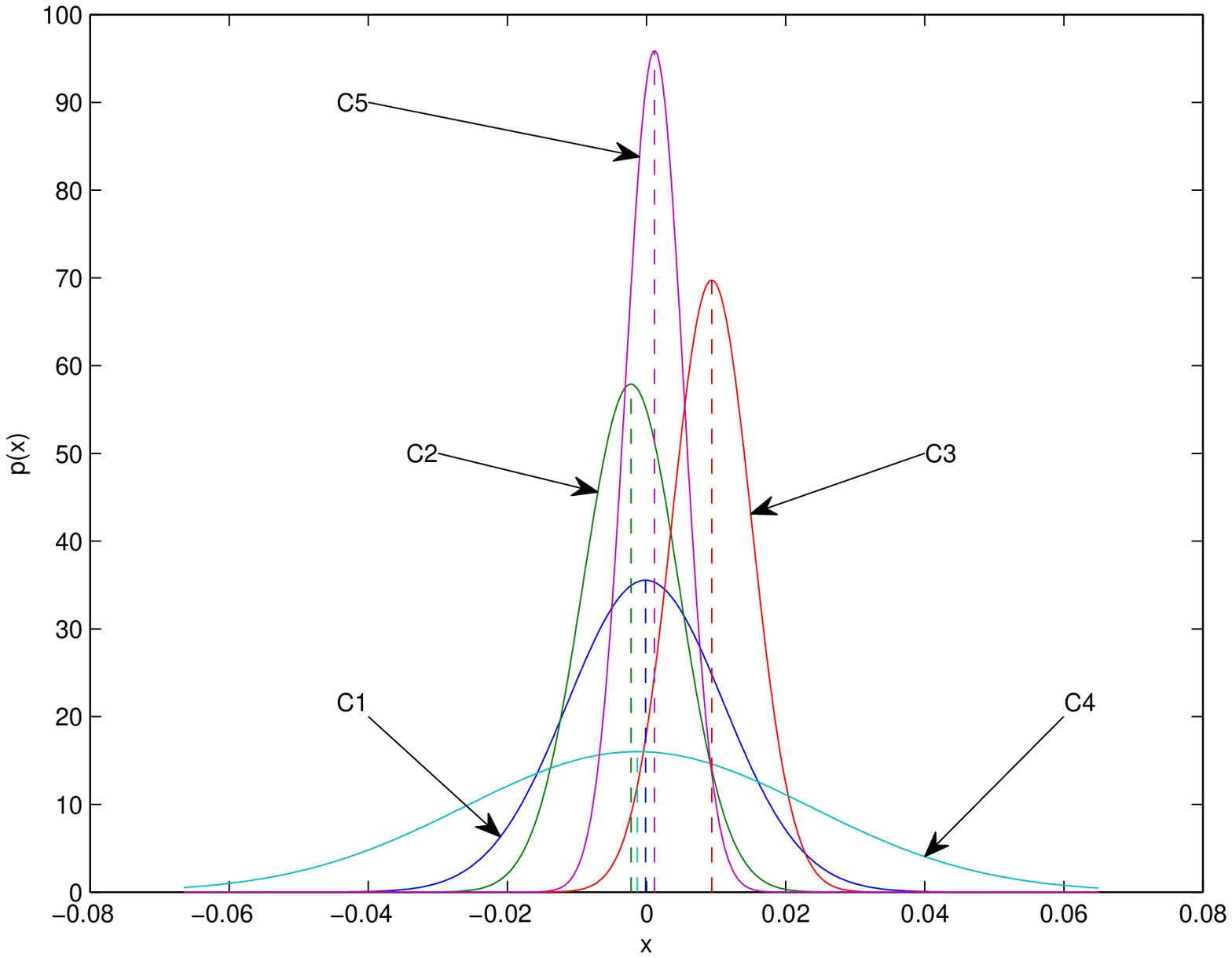}    \caption{Symmetric}
  \end{subfigure}
  ~
  \begin{subfigure}[b]{0.48\linewidth}
    \psfrag{C1}[c][r]{\scriptsize C1}
    \psfrag{C2}[c][l]{\scriptsize C2}
    \psfrag{C3}[c][r]{\scriptsize C3}
    \psfrag{C4}[c][l]{\scriptsize C4}
    \psfrag{C5}[c][l]{\scriptsize C5}
    \psfrag{x}[c][c]{\small $x$}
    \psfrag{p(x)}[b][c]{\small $p(x)$}
\includegraphics[width=\linewidth]{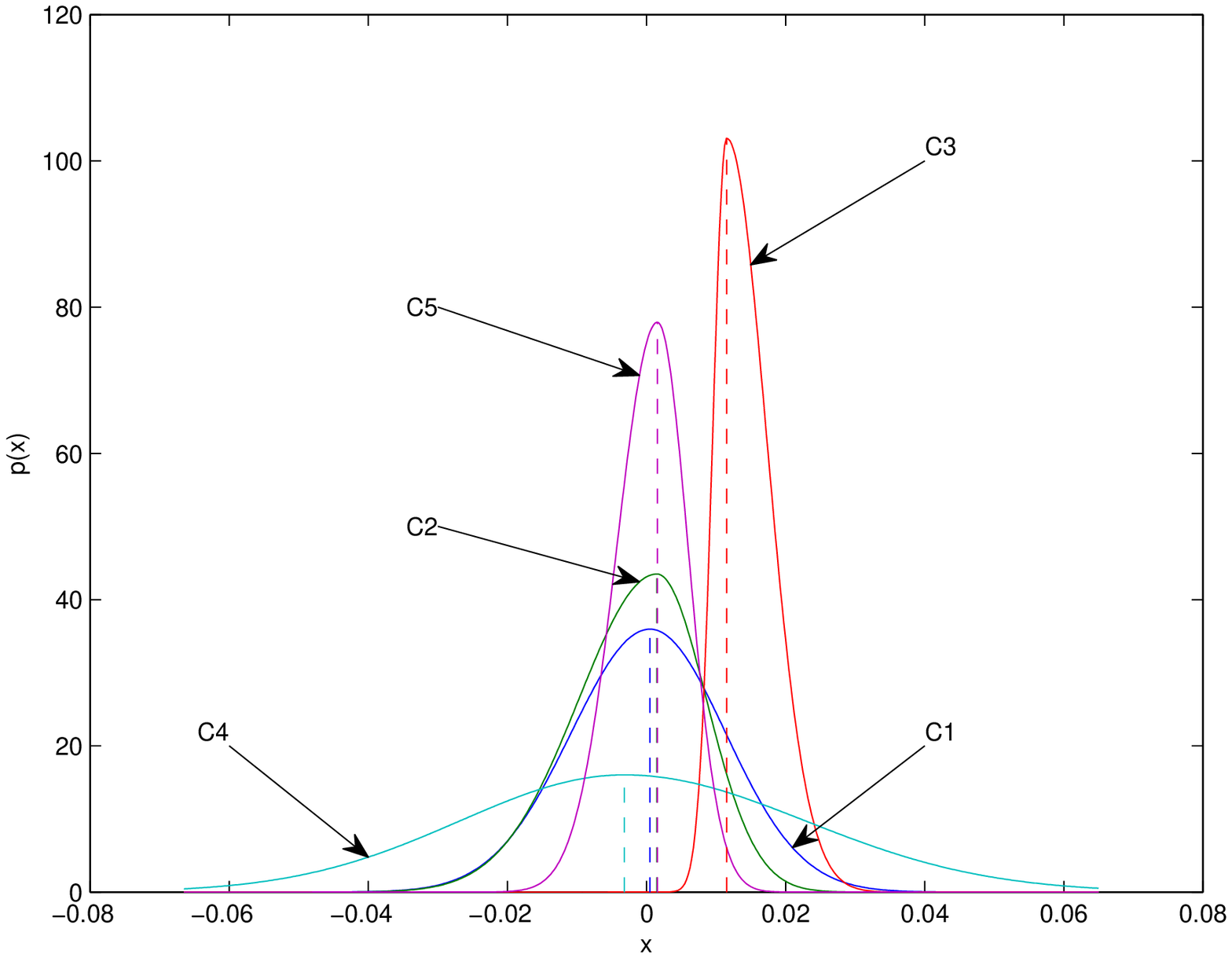}    \caption{Asymmetric}
  \end{subfigure}
  \caption{Probability density function for each emission distribution.}
  \label{fig:application:values}
\end{figure*}

Figure~\ref{fig:application:values} shows the emission distributions for each
HMM, with the mode dashed to highlight the asymmetry. While some asymmetries are
more subtle, like in components C4 ($p=0.476$) and C2 ($p=0.510$), others are
more noticeable, like C1 ($p=0.612$) and C5 ($p=0.570$). In special, the
component C3 has the largest asymmetry of all, with $p=0.260$.

Since the shape of the base distribution, in this case the normal distribution,
has been preserved in each side, the weight for each case can be used to provide
some additional insight into the state. For example, the state associated with
the component C3 is considerably certain that the index will rise ($x > 0$),
which none of the emissions in the symmetric case indicates.

While the increased likelihood and the presence of asymmetry are expected from
using a more general version of the distributions, other interesting and
potentially useful results appear when we analyze the distribution over states.

\begin{figure*}[t]
  \centering
  \begin{subfigure}[b]{0.48\linewidth}
    \psfrag{Source}[b][c]{\small Source}
    \psfrag{Target}[t][c]{\small Target}
\includegraphics[width=\linewidth]{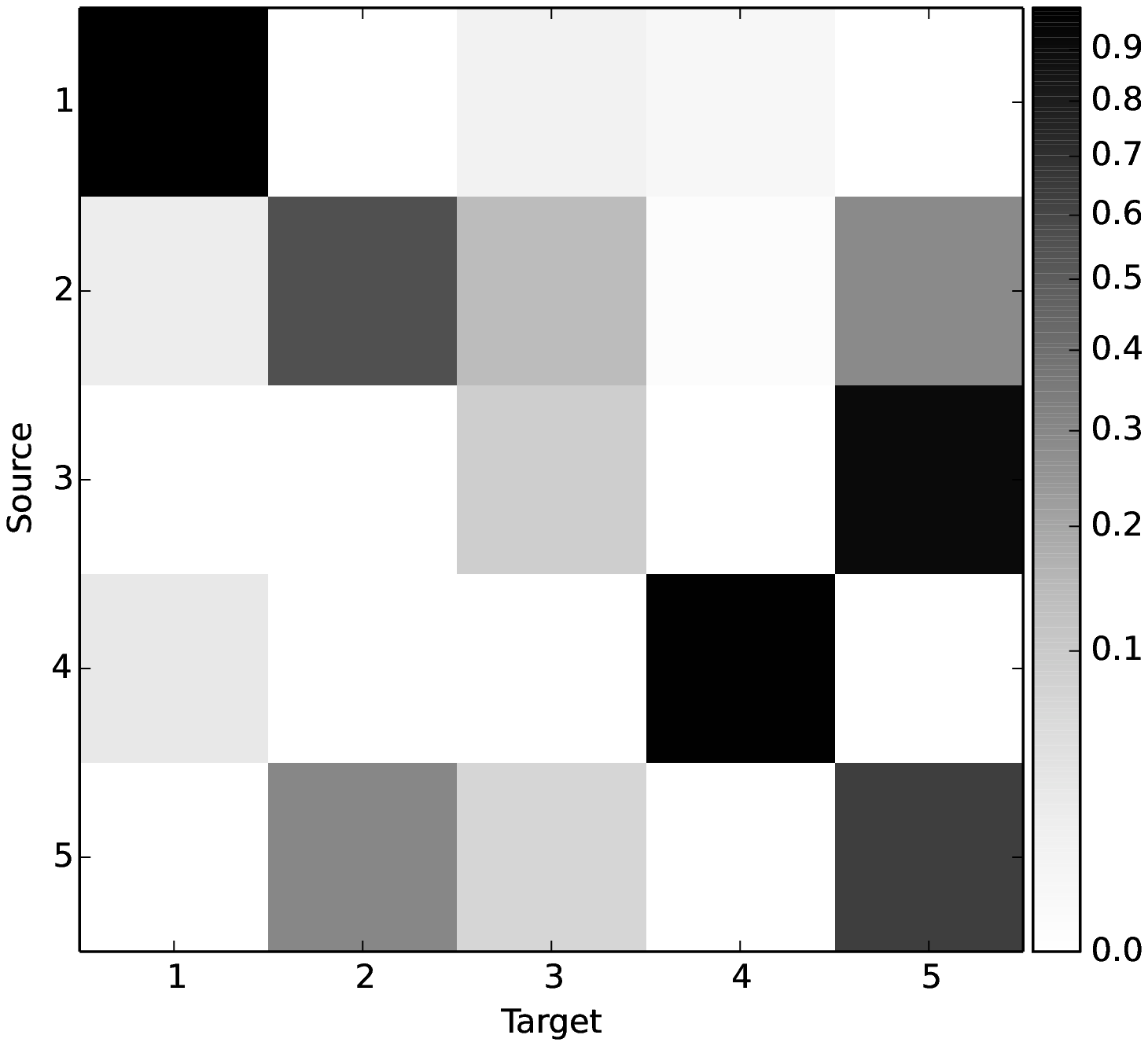}    \caption{Symmetric}
  \end{subfigure}
  ~
  \begin{subfigure}[b]{0.48\linewidth}
    \psfrag{Source}[b][c]{\small Source}
    \psfrag{Target}[t][c]{\small Target}
\includegraphics[width=\linewidth]{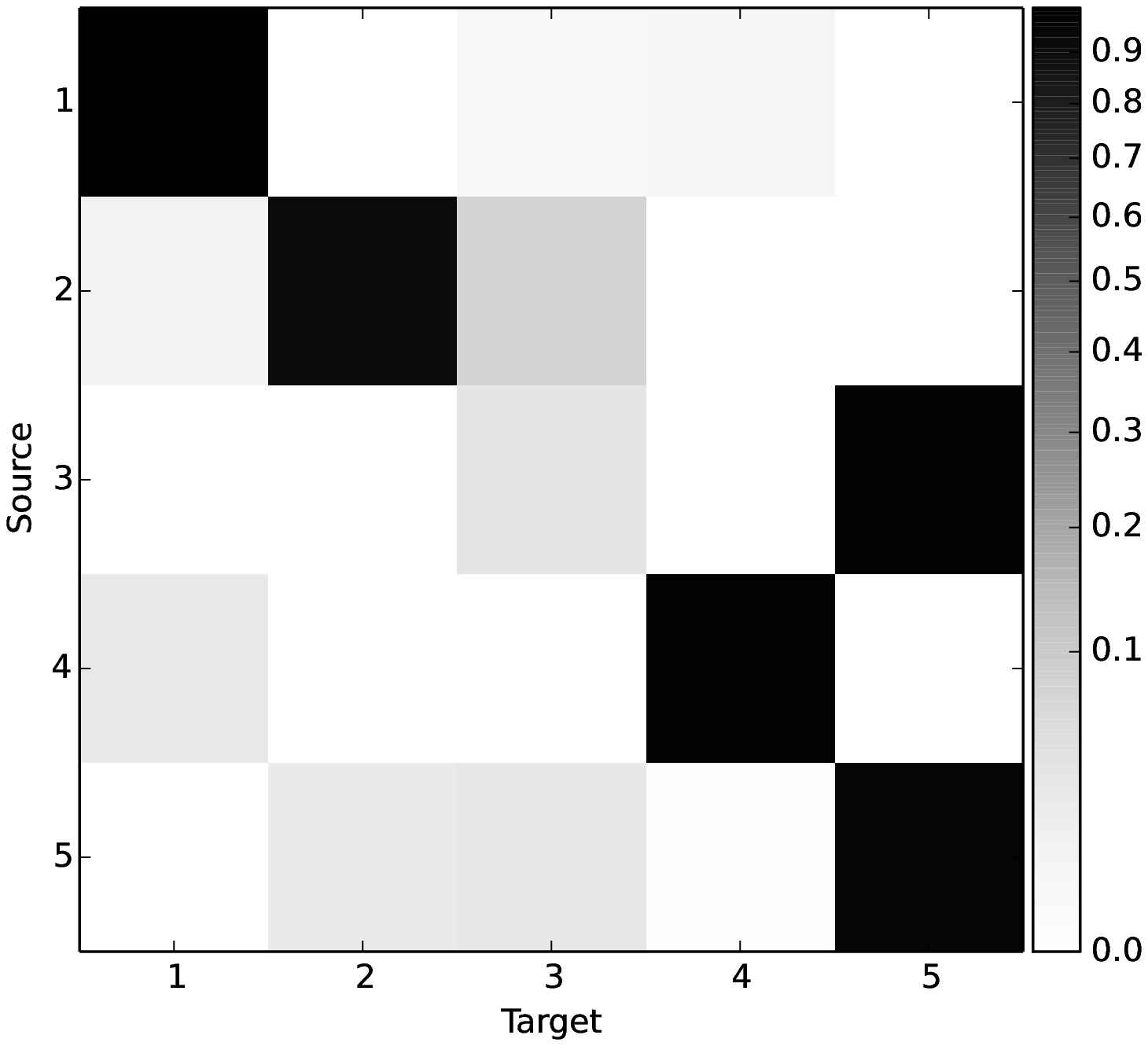}    \caption{Asymmetric}
  \end{subfigure}
  \caption{Transition probabilities of the HMM states using the symmetric and
  asymmetric distributions, with darker having higher probability.}
  \label{fig:application:transition}
\end{figure*}

When we evaluate the transition probabilities for each state, shown in
Figure~\ref{fig:application:transition}, it becomes very clear that the
transitions for the asymmetric version are usually much less ambiguous. To
evaluate this quantitatively, Table~\ref{tab:application:entropy} shows the
entropy of the transitions out of each state, with the maximum entropy being
given by $\log_2 5 = 2.3219$ bits.

Except for the fourth state, which suffered a minor increase in entropy of
$1.2\%$ and had no noticeable difference in
Figure~\ref{fig:application:transition}, all other transitions reduced the
entropy considerably, from $32.6\%$ to $70.2\%$, with clear differences in the
transition.

\begin{figure*}[t]
  \centering
  \begin{subfigure}[b]{0.48\linewidth}
    \psfrag{Histogram with 5 components and missing data}[c][c]{\small Histogram
    with missing data}
    \psfrag{Histogram with 5 components}[c][c]{\small Histogram without missing
    data}
    \psfrag{Normalized entropy}[c][c]{\footnotesize Normalized entropy}
\includegraphics[width=\linewidth]{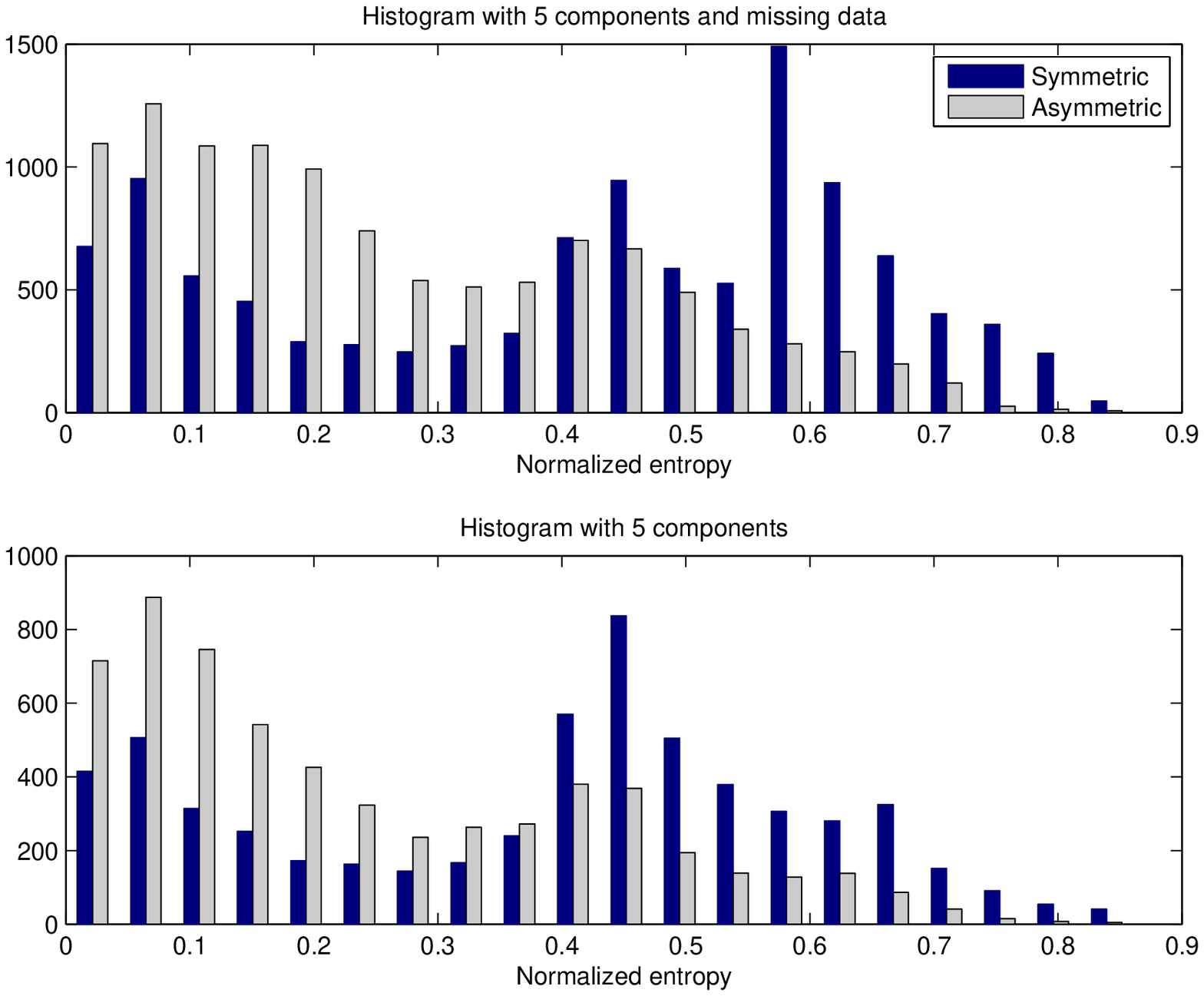}    \caption{Entropy histogram}
    \label{fig:application:entropy:hist}
  \end{subfigure}
  ~
  \begin{subfigure}[b]{0.48\linewidth}
    \psfrag{Normalized entropy asymmetric}[b][c]{\small Normalized entropy
    asymmetric}
    \psfrag{Normalized entropy symmetric}[t][c]{\small Normalized entropy
    symmetric}
\includegraphics[width=\linewidth]{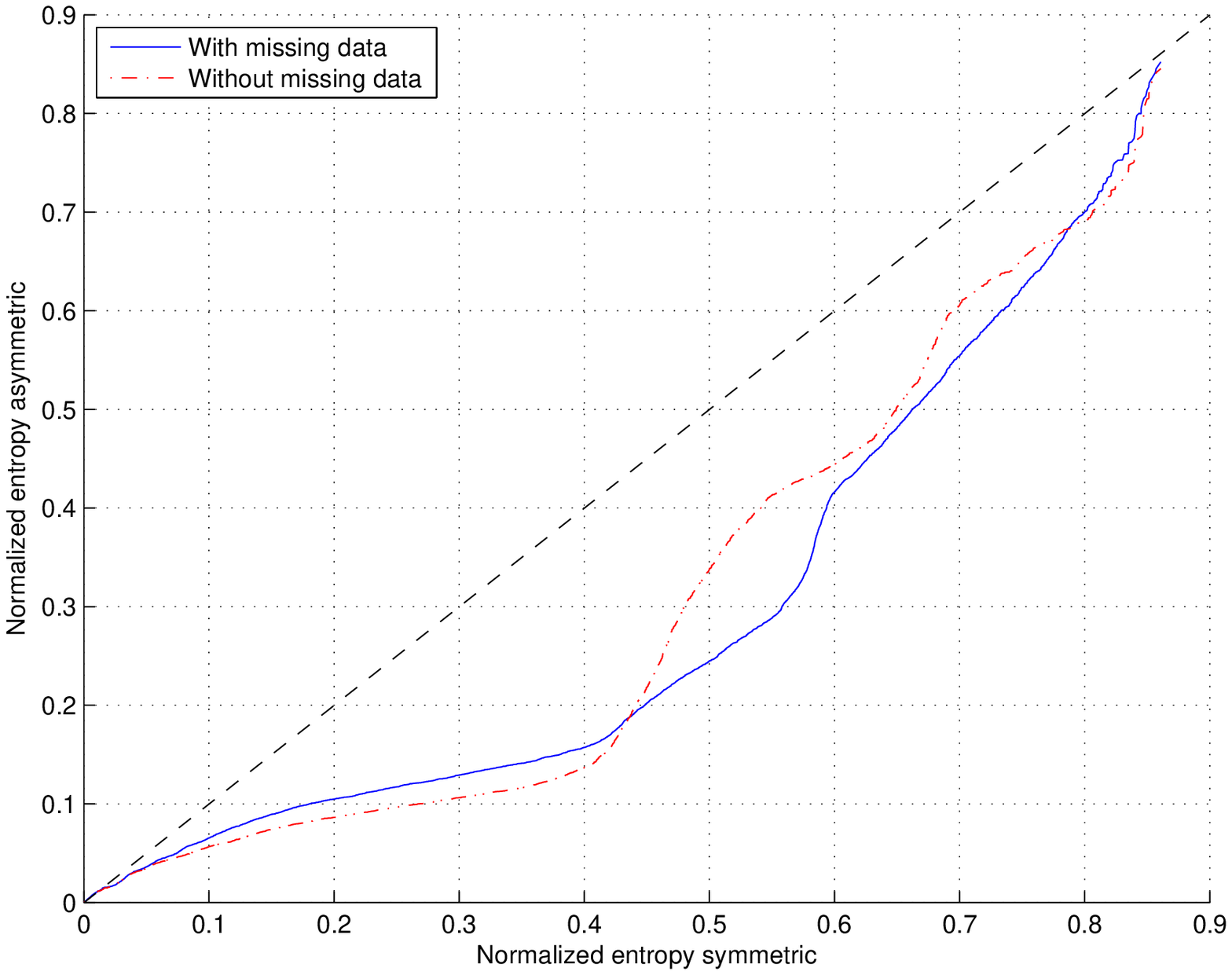}    \caption{Entropy QQ plot}
    \label{fig:application:entropy:qq}
  \end{subfigure}
  \caption{Normalized entropy of the HMM states with and without the missing
  data.}
  \label{fig:application:entropy}
\end{figure*}

This reduced entropy also occurs in the states themselves, as shown in
Figure~\ref{fig:application:entropy}. Figure~\ref{fig:application:entropy:hist}
shows the histogram of normalized entropies, which is the entropy divided by the
maximum entropy, for both HMMs and considering the state of missing data or not.
In both cases, the asymmetric version has considerably more states with lower
entropy than the symmetric version. Note also that the asymmetric version
appears to suffer less from missing data, while the symmetric version has a
spike around $0.6$ that does not occur without considering these states.

To emphasize the difference, Figure~\ref{fig:application:entropy:qq} shows the
entropy QQ plot, which is composed of plotting the normalized entropy quantiles
of each HMM's states, with the dashed line representing the identity. From this
figure, we note that the symmetric HMM's states indeed have higher entropy than
the ones from the asymmetric, with the first reaching normalized entropy 0.4
before the latter gets $0.2$, and a quantile with asymmetric distributions
almost always has less entropy than its equivalent symmetric, with the only
exceptions being the first few quantiles with very low entropy. Additionally,
this figure also shows that the curves that considers the missing data is close
to the one that does not, also indicating that the asymmetric version has good
performance despite this lack of information.

\section{Conclusion}
\label{sec:conclusion}
In this paper, we have introduced the concept of a constrained mixture and
provided two examples of how it can be used with the Laplace and normal
distributions to create new asymmetric distributions. The new distributions were
shown to generalize their underlying distribution while keeping important
properties, such as belonging to the exponential family and having maximum
likelihood estimates and conjugate priors with known closed-form expressions.
Moreover, the distributions were shown to have an inherent regularization term,
that is, a regularization that comes directly from the likelihood and not an
imposed cost, that penalizes the asymmetry, such that the distribution avoids
unnecessarily deforming the symmetric underlying distribution.

One of the new distributions, the asymmetric normal distribution, was compared
to the symmetric version in a regression example with asymmetric noise. This
allowed a better understanding of how the asymmetric distributions operate and
showed that, since the symmetric versions are particular cases of the asymmetric
distributions, the asymmetric ones must have higher likelihood.

The asymmetric and symmetric normal distributions were also compared when used
for emissions in a hidden Markov model (HMM) for a stock index. Results show
that, as one would expect, the additional flexibility of the asymmetry allowed
the distribution to better fit the data, providing increased likelihood and with
larger differences as more states were used.

A positive consequence of this flexibility and better fitting was additional
certainty in the states and their transitions. We have shown that, when the HMM
had 5 states, most probability distributions over the states had a considerable
reduction in their entropy even when missing data is considered. Moreover,
although one transition distribution, which already exhibited low entropy, had
its entropy increased by $1.2\%$, all other transitions had reduced entropy,
losing from $32.6\%$ to $70.2\%$ of their values, and the largest transition
entropy is less than $20\%$ of the maximum entropy, compared to $64.5\%$ for the
symmetric version.

Future investigations involve analyzing if it is possible to know the maximum
likelihood estimates and conjugate priors and their closed-form expressions for
the Laplace and normal distributions when the domain split does not occur at the
mode. If so, the effect of using the constrained mixture in other distributions
of the exponential family and the use of multiple segments should be
investigated. Besides this theoretical research, the use of asymmetry to
characterize loss functions in machine learning is of interest, since it can
make the system focus more on predicting low or high values.

\acks{The authors would like to thank CNPq for the financial support.}

\appendix
\section{Proof of Theorem~\ref{thm:laplace}}
\begin{proof}
  Constraint~\ref{const:continuity} is trivially satisfied, since both sides
  converge to $\beta$. From Constraint~\ref{const:normalization}, one has that
  \begin{align*}
    \int_{-\infty}^\infty \psi(x;\mu,\lambda,p) \dx
    &= \beta \left(
      \int_{-\infty}^\mu \exp(\lambda\alpha^{-1}(x-\mu)) \dx
      +\int_\mu^\infty \exp(-\lambda\alpha(x-\mu)) \dx
    \right)
    \\
    &=\beta \left(\frac{\alpha}{\lambda} + \frac{1}{\lambda\alpha}\right)
    = \beta \frac{\alpha^2+1}{\lambda \alpha} = 1,
  \end{align*}
  which is satisfied by the definition of $\beta$.

  From Constraint~\ref{const:weight}, one has that
  \begin{align*}
    \int_{-\infty}^\mu \psi(x;\mu,\lambda,p) \dx
    &= \beta \int_{-\infty}^\mu \exp(\lambda\alpha^{-1}(x-\mu)) \dx
    \\
    &= \beta \frac{\alpha}{\lambda}
    = \frac{\lambda \alpha}{\alpha^2+1} \frac{\alpha}{\lambda}
    = \frac{\alpha^2}{\alpha^2+1}
    = \frac{\left(\frac{p}{1-p}\right)}{\left(\frac{p}{1-p}\right)+1}
    = p,
  \end{align*}
  which is satisfied by the definition of $\alpha$ and $\beta$.

  Finally, to satisfy Constraint~\ref{const:mixture}, let
  $\Theta_-(\mu,\lambda,p) = [\mu, \lambda\alpha^{-1}]$ and
  $\Theta_+(\mu,\lambda,p) = [\mu, \lambda\alpha]$. Then
  \begin{align*}
    \psi(x;\mu,\lambda,p)
    &= \beta \exp(-\lambda\alpha (x-\mu))\I[x \ge \mu]
    +\beta \exp(-\lambda\alpha^{-1} (\mu-x))\I[x < \mu]
    \\
    &= \frac{2 \beta}{\lambda \alpha} \phi_+(x;\Theta_+(\cdot))
    + \frac{2 \beta \alpha}{\lambda} \phi_-(x;\Theta_-(\cdot))
    \\
    &=
    \frac{2}{\alpha^2+1} \phi_+(x;\Theta_+(\cdot)) +
    \frac{2 \alpha^2}{\alpha^2+1} \phi_-(x;\Theta_-(\cdot))
    \\
    &=
    2(1-p) \phi_+(x;\Theta_+(\cdot)) + 2 p \phi_-(x;\Theta_-(\cdot)),
  \end{align*}
  which sets $Z = 1/2$.
\end{proof}
\section{Proof of Theorem~\ref{thm:normal}}
\begin{proof}
  Constraint~\ref{const:continuity} is trivially satisfied, since both sides
  converge to $\beta$. From Constraint~\ref{const:normalization}, one has that
  \begin{align*}
    \int_{-\infty}^\infty \psi(x;\mu,\sigma,p) \dx
    &= \beta \left(
      \int_{-\infty}^\mu \Phi\left(\frac{x-\mu}{\sigma \alpha}\right) \dx
      +\int_\mu^\infty \Phi\left(\frac{x-\mu}{\sigma \alpha^{-1}}\right) \dx
    \right)
    \\
    &= \frac{\beta}{\sqrt{2\pi}} \left(
      \int_{-\infty}^\mu \exp\left(-\frac{{(x-\mu)}^2}{2 \sigma^2
      \alpha^2}\right) \dx
      +\int_\mu^\infty \exp\left(-\frac{{(x-\mu)}^2}{2 \sigma^2
      \alpha^{-2}}\right) \dx
    \right)
    \\
    &= \frac{\beta}{\sqrt{2\pi}} \left(
      \int_\mu^\infty \exp\left(-\frac{{(x-\mu)}^2}{2 \sigma^2
      \alpha^2}\right) \dx
      +\int_\mu^\infty \exp\left(-\frac{{(x-\mu)}^2}{2 \sigma^2
      \alpha^{-2}}\right)
      \dx
    \right)
    \\
    &= \frac{\beta}{\sqrt{2\pi}} \frac{\sqrt{\pi}}{2} \left(
      \sqrt{2}\alpha\sigma \left.\text{erf}
      \left(\frac{x-\mu}{\sqrt{2}\sigma\alpha}\right)\right|_\mu^\infty
      +\frac{\sqrt{2}\sigma}{\alpha} \left.\text{erf}
      \left(\frac{x-\mu}{\sqrt{2}\sigma\alpha^{-1}}\right)\right|_\mu^\infty
    \right)
    \\
    &= \frac{\beta}{2} \left(\alpha \sigma + \frac{\sigma}{\alpha}\right)
    = \frac{\beta\sigma}{2} \left(\frac{\alpha^2+1}{\alpha}\right) = 1,
  \end{align*}
  which is satisfied by the definition of $\beta$, where $\text{erf}(\cdot)$ is
  the error function.

  From Constraint~\ref{const:weight}, one has that
  \begin{align*}
    \int_{-\infty}^\mu \psi(x;\mu,\sigma,p) \dx
    &= \beta \int_{-\infty}^\mu \Phi\left(\frac{x-\mu}{\sigma\alpha}\right)
    \dx
    \\
    &= \frac{\beta}{\sqrt{2\pi}} \int_{-\infty}^\mu
    \exp\left(-\frac{{(x-\mu)}^2}{2\sigma^2\alpha^2}\right) \dx
    \\
    &= \frac{\beta}{\sqrt{2\pi}} \frac{\sqrt{\pi}}{2}
    \sqrt{2}\alpha\sigma \left.\text{erf}
    \left(\frac{x-\mu}{\sqrt{2}\sigma\alpha}\right)\right|_\mu^\infty
    \\
    &= \frac{\beta}{2} \alpha \sigma
    = \frac{2 \alpha}{2 \sigma (\alpha^2+1)} \alpha \sigma
    = \frac{\left(\frac{p}{1-p}\right)}{\left(\frac{p}{1-p}\right)+1}
    = \frac{\alpha^2}{\alpha^2+1}
    = p,
  \end{align*}
  which is satisfied by the definition of $\alpha$ and $\beta$.

  Finally, to satisfy Constraint~\ref{const:mixture}, let
  $\Theta_-(\mu,\lambda,p) = [\mu, \sigma\alpha]$ and $\Theta_+(\mu,\lambda,p) =
  [\mu, \sigma\alpha^{-1}]$. Then
  \begin{align*}
    \psi(x;\mu,\sigma,p)
    &=
    \beta \Phi\left(\frac{x-\mu}{\sigma\alpha^{-1}}\right)\I[x \ge \mu] +
    \beta \Phi\left(\frac{x-\mu}{\sigma\alpha}\right)\I[x < \mu]
    \\
    &= \beta \sigma \alpha^{-1} \phi_+(x;\Theta_+(\cdot)) +
    \beta \sigma \alpha \phi_+(x;\Theta_+(\cdot))
    \\
    &=
    \frac{2}{\alpha^2+1} \phi_+(x;\Theta_+(\cdot)) +
    \frac{2 \alpha^2}{\alpha^2+1} \phi_-(x;\Theta_-(\cdot))
    \\
    &=
    2(1-p) \phi_+(x;\Theta_+(\cdot)) + 2 p \phi_-(x;\Theta_-(\cdot)),
  \end{align*}
  which sets $Z = 1/2$.
\end{proof}
\section{Proof of lemma for Theorem~\ref{thm:laplace_opt}}
\begin{lemma}
  \label{lem:laplace_convex}
  Let $p \in (0,1)$ and $S = \{s_i\}, i \in \{1,2,\ldots,N\}, s_i \in \R$, be
  given. Let the pdf of the asymmetric Laplace distribution be given by
  Equation~\eqref{eq:asym_laplace_pdf}. Then the function
  \begin{equation*}
    \gamma(\mu) = \alpha \sum_{s_i \in S_+} (s_i-\mu) -
    \alpha^{-1} \sum_{s_i \in S_-} (s_i-\mu),
  \end{equation*}
  where $S_- = \{s_i \in S | s_i < \mu\}$, $S_+ = \{s_i \in S | s_i >
  \mu\}$, and $\alpha = \sqrt{\frac{p}{1-p}}$, is convex.

  Furthermore, let $\mu,\mu' \in \R,\mu < \mu'$. If there is some $s_i \in S$
  such that $\mu < s_i < \mu'$, then $\gamma(t\mu+(1-t)\mu') <
  t\gamma(\mu)+(1-t)\gamma(\mu')$ for all $t \in (0,1)$.
\end{lemma}
\begin{proof}
  Let $t \in [0,1]$ and $t' = 1-t$. Let $\mu, \mu' \in \R, \mu \le \mu'$. Let
  $\eta_i$ be a variable associated with sample $s_i$, such that
  \begin{equation*}
    \eta_i = \alpha\I[s_i \ge t\mu + t'\mu'] - \alpha^{-1}\I[s_i < t\mu +
    t'\mu'],
  \end{equation*}
  where $\I[\cdot]$ is the indicator function. Since $\alpha > 0$, we have that
  $\eta_i - \alpha \le 0$ and $\eta_i + \alpha^{-1} \ge 0$, and $\eta_i - \alpha
  = 0 \Leftrightarrow \eta_i + \alpha^{-1} \ne 0$.

  Let
  $S_- = \{s_i \in S | s_i < \mu\}$,
  $S_+ = \{s_i \in S | s_i \ge \mu\}$,
  $S_-' = \{s_i \in S | s_i < \mu'\}$,
  $S_+' = \{s_i \in S | s_i \ge \mu'\}$, $S^* = S_+ \cap S_-'$. Then
  \begin{align*}
    &\gamma(t\mu+t'\mu')
    \\
    &= \alpha \sum_{s_i \in S'_+} (s_i - t\mu - t'\mu')
    - \alpha^{-1} \sum_{s_i \in S_-} (s_i - t\mu - t'\mu')
    + \sum_{s_i \in S^*} \eta_i (s_i - t\mu - t'\mu')
    \\
    &=
    t \gamma(\mu) + t' \gamma(\mu') +
    \sum_{s_i \in S^*}
    \left(
        \eta_i (s_i - t\mu - t'\mu')
        - t \alpha(s_i - \mu)
        + t' \alpha^{-1} (s_i - \mu')
    \right)
    \\
    &=
    t \gamma(\mu) + t' \gamma(\mu') +
    \sum_{s_i \in S^*}
    \left(
        t\underbrace{(s_i-\mu)}_{\ge 0}\underbrace{(\eta_i-\alpha)}_{\le 0}
        + t'\underbrace{(s_i-\mu')}_{<
        0}\underbrace{(\eta_i+\alpha^{-1})}_{\ge 0}
    \right)
    \\
    &\le t\gamma(\mu) + t'\gamma(\mu'),
  \end{align*}
  which proves that $\gamma(\mu)$ is a convex function.

  Moreover, if there is some $\mu < s_i < \mu'$, then $s_i \in S^*$ and either
  $\eta_i - \alpha < 0$ or $\eta_i + \alpha^{-1} > 0$, so that
  $\gamma(t\mu+t'\mu') < t\gamma(\mu)+t'\gamma(\mu')$ for all $t \in (0,1)$.
\end{proof}
\section{Proof of Theorem~\ref{thm:laplace_opt}}
\begin{proof}
  From Equation~\eqref{eq:laplace_likelihood}, one can see that $\mu$ can be
  optimized independently from the value of $\lambda$. Let $\gamma(\mu)$ be
  defined as in Lemma~\ref{lem:laplace_convex}, such that
  \begin{equation*}
    \ln \mathcal L = C + |S| \ln \lambda - \lambda \gamma(\mu),
  \end{equation*}
  where $C$ is a constant. Therefore, the value $\mu^*$ that minimizes
  $\gamma(\mu)$ is the maximum likelihood estimator. The function $\gamma(\mu)$
  can be rewritten as
  \begin{equation*}
    \gamma(\mu) = \alpha \sum_{s_i \in S_+}|s_i-\mu| + \alpha^{-1} \sum_{s_i \in
    S_-} |s_i-\mu|,
  \end{equation*}
  which is associated with the log-likelihood of the weighted scale-free Laplace
  distribution, whose maximum likelihood estimate $\mu^*$ is given by the
  weighted median~\citep{weightedmedian} with samples in $S_-$ and $S_+$
  weighting $\alpha^{-1}$ and $\alpha$, respectively.

  For $\lambda$, the optimal value is given by:
  \begin{equation*}
    \frac{\partial \ln \mathcal L}{\partial \lambda} = \frac{|S|}{\lambda} -
    \gamma(\mu) = 0,
  \end{equation*}
  which solves for the value provided by the theorem.

  From Lemma~\ref{lem:laplace_convex}, we also know that there is no sample
  between two optima $\mu^*_1$ and $\mu^*_2$, $\mu^*_1<\mu^*_2$, of
  $\gamma(\mu)$, or there would be some $t \in (0,1)$ such that $\gamma(t\mu^*_1
  + (1-t)\mu^*_2) < t \gamma(\mu^*_1) + (1-t)\gamma(\mu^*_2) <
  \max\{\gamma(\mu^*_1), \gamma(\mu^*_2)\}$, which contradicts the optimality of
  $\mu^*_1$ or $\mu^*_2$.
\end{proof}
\section{Proof of Theorem~\ref{thm:laplace_prior}}
\begin{proof}
  Using Equation~\eqref{eq:prior}, we have that the prior can be written as:
  \begin{align*}
    f(p, \lambda; \chi,\nu)
    &=
    C \exp\left(-\lambda \alpha \chi_1 -\lambda \alpha^{-1} \chi_2 + \nu \ln
    \beta\right)
    \\
    &=
    C \exp\left(
      -\lambda \alpha \chi_1 -\lambda \alpha^{-1} \chi_2
      + \nu \left(\ln \lambda + \frac{1}{2} \left(\ln p + \ln(1-p)\right)\right)
    \right)
    \\
    &= C \exp(-\lambda \alpha \chi_1 -\lambda \alpha^{-1} \chi_2) \lambda^\nu
    p^{\nu/2} (1-p)^{\nu/2}
    \\
    &= C \exp(-\lambda \alpha \chi_1 -\lambda \alpha^{-1} \chi_2) \lambda^\nu
    B(p;\nu')
    \\
    &= C_1 \exp(-\lambda \alpha \chi_1 -\lambda \alpha^{-1} \chi_2) (\lambda
    \alpha)^{\nu/2} (\lambda \alpha^{-1})^{\nu/2} B(p;\nu')
    \\
    &= G(\lambda \alpha; \nu', \chi_1)
    G(\lambda \alpha^{-1}; \nu', \chi_2)
    B(p;\nu'),
  \end{align*}
  where $\nu' = 1+\nu/2$.
\end{proof}
\section{Proof of lemma for Theorem~\ref{thm:normal_opt}}
\begin{lemma}
  \label{lem:normal_convex}
  Let $p \in (0,1)$ and $S = \{s_i\}, i \in \{1,2,\ldots,N\}, s_i \in \R$, be
  given. Let the pdf of the asymmetric normal distribution be given by
  Equation~\eqref{eq:asym_normal_pdf}. Then the function
  \begin{equation*}
    \gamma(\mu) = \alpha^{-2} \sum_{s_i \in S_-} {(s_i-\mu)}^2 +
    \alpha^2 \sum_{s_i \in S_+} {(s_i-\mu)}^2,
  \end{equation*}
  where $S_- = \{s_i \in S | s_i < \mu\}$, $S_+ = \{s_i \in S | s_i \ge
  \mu\}$, and $\alpha = \sqrt{\frac{p}{1-p}}$, is strictly convex.
\end{lemma}
\begin{proof}
  Let $f(x) \colon \R \to \R$ be a function and $f^{(n)}(x)$ its $n$-th
  derivative. If $f(x)$ and $f'(x)$ are continuous and $f''(x) > 0$ for all $x$,
  then $f(x)$ is strictly convex.

  For fixed $S_-$ and $S_+$, $\gamma(\mu)$ is a strictly convex quadratic
  function of $\mu$. If $\gamma(\mu)$ is continuously differentiable and its
  derivative is monotonically increasing for variables $S_-$ and $S_+$, then
  $\gamma(\mu)$ is strictly convex.

  Let $s_* = \min S_+$. The limit $\mu \to s_*$ is given by:
  \begin{align*}
    \lim_{\mu \to s_*^-} \gamma(\mu)
    &= \lim_{\mu \to s_*^-}
    \alpha^{-2} \sum_{s_i \in S_-} {(s_i-\mu)}^2 +
      \alpha^2 \sum_{s_i \in S_+} {(s_i-\mu)}^2
    \\
    &=
    \left(
      \alpha^{-2} \sum_{s_i \in S_- \cup \{s_*\}} {(s_i-s_*)}^2 +
      \alpha^2 \sum_{s_i \in S_+ \backslash \{s_*\}} {(s_i-s_*)}^2
    \right)
    \\
    &= \lim_{\mu \to s_*^+} \gamma(\mu),
  \end{align*}
  which proves that $\gamma(\mu)$ is continuous. Its derivative is given by:
  \begin{equation}
    \label{eq:normal_gamma_derivative}
    \gamma'(\mu) =
    -2 \alpha^2 \sum_{s_i \in S_+} (s_i-\mu)
    -2\alpha^{-2} \sum_{s_i \in S_-} (s_i-\mu),
  \end{equation}
  and we can prove that it is continuous using the same method as before.

  Since $\gamma''(\mu) \ge 2|S| \min\{\alpha^2, \alpha^{-2}\} > 0$, the
  derivative $\gamma'(\mu)$ is monotonically increasing and $\gamma(\mu)$ is
  strictly convex.
\end{proof}
\section{Proof of Theorem~\ref{thm:normal_opt}}
\begin{proof}
  From Equation~\eqref{eq:normal_likelihood}, one can see that $\mu$ can be
  optimized independently of the value of $\lambda$. Let $\gamma(\mu)$ be
  defined as in Lemma~\ref{lem:normal_convex}, such that
  \begin{equation*}
    \ln \mathcal L = C - |S| \ln \sigma - \frac{1}{2\sigma^2} \gamma(\mu),
  \end{equation*}
  where $C$ is a constant. Therefore, the value
  $\mu^*$ that minimizes $\gamma(\mu)$ is the maximum likelihood estimator. And,
  since $\gamma(\mu)$ is strictly convex, this value is unique.

  From the first order optimality condition, we can solve
  Equation~\eqref{eq:normal_gamma_derivative} to find the optimal $\mu^*$ stated
  in the theorem. For $\sigma$, the optimal value is given by:
  \begin{equation*}
    \frac{\partial \ln \mathcal L}{\partial \sigma} = -\frac{|S|}{\sigma} +
    \frac{1}{\sigma^3} \gamma(\mu) = 0,
  \end{equation*}
  which solves for the value provided by the theorem.
\end{proof}
\section{Proof of Theorem~\ref{thm:normal_prior}}
\begin{proof}
  Using Equation~\eqref{eq:prior}, we have that the prior can be written as:
  \begin{align*}
    f(p, \sigma; \chi,\nu)
    &=
    C \exp\left(-\frac{\chi_1}{2 \sigma^2 \alpha^{-2}}
      -\frac{\chi_2}{2 \sigma^2 \alpha^2} + \nu \ln \beta\right)
    \\
    &=
    C \exp\left(
        -\frac{\chi_1}{2 \sigma^2 \alpha^{-2}}
        -\frac{\chi_2}{2 \sigma^2 \alpha^2}
        + \nu \left(\ln 2 -\ln \sigma + \frac{1}{2} \left(\ln p +
        \ln(1-p)\right)\right)
    \right)
    \\
    &=
    C_1 \exp\left(
      -\frac{\chi_1}{2 \sigma^2 \alpha^{-2}}
      -\frac{\chi_2}{2 \sigma^2 \alpha^2}
    \right)
    \sigma^{-\nu}
    p^{\nu/2}(1-p)^{\nu/2}
    \\
    &=
    C_1 \exp\left(
      -\frac{\chi_1}{2 \sigma^2 \alpha^{-2}}
      -\frac{\chi_2}{2 \sigma^2 \alpha^2}
    \right)
    \sigma^{-\nu}
    B(p,\nu_1)
    \\
    &=
    C_1 \exp\left(
      -\frac{\chi_1}{2 \sigma^2 \alpha^{-2}}
      -\frac{\chi_2}{2 \sigma^2 \alpha^2}
    \right)
    (\sigma \alpha)^{-\nu/2}
    (\sigma \alpha^{-1})^{-\nu/2}
    B(p;\nu_1)
    \\
    &=
    Ig(\sigma^2 \alpha^2; \nu_2, \chi_2')
    Ig(\sigma^2 \alpha^{-2}; \nu_2, \chi_1')
    B(p;\nu_1)
  \end{align*}
  where $\nu_1 = 1+\nu/2$, $\nu_2 = \nu/4-1$ and $\chi_i' = \chi_i/2$.
\end{proof}

\bibliography{paper}

\end{document}